\documentclass[11pt]{article}
\usepackage[margin=1in]{geometry}

\usepackage[utf8]{inputenc}
\usepackage[T1]{fontenc}
\usepackage[sc]{mathpazo}

\RequirePackage{algorithm}
\RequirePackage{algorithmic}
\RequirePackage{natbib}

\usepackage{titling} % For better title control

\usepackage{graphicx, enumitem, mathrsfs} % Required for inserting images
\usepackage{booktabs} % for professional tables
\usepackage{lipsum}
\usepackage{hyperref}

\usepackage{amsmath}
\usepackage{amssymb}
\usepackage{mathtools}
\usepackage{amsthm}
\usepackage{arydshln}
\usepackage{caption}
\usepackage{subcaption}

\usepackage{thmtools}
\usepackage{thm-restate}
\usepackage{hyperref}
\usepackage{url}
\usepackage{xspace}
\usepackage{graphicx}
\usepackage{subcaption}
\usepackage{wrapfig}
\usepackage{booktabs}
\usepackage{amsthm}
\usepackage[capitalize,noabbrev]{cleveref}

\theoremstyle{plain}
\newtheorem{theorem}{Theorem}[section]
\newtheorem{proposition}[theorem]{Proposition}
\newtheorem{lemma}[theorem]{Lemma}
\newtheorem{corollary}[theorem]{Corollary}
\theoremstyle{definition}

\newtheorem{assumption}[theorem]{Assumption}
\theoremstyle{remark}
\newtheorem{remark}[theorem]{Remark}

\usepackage{xcolor, soul}

\definecolor{darkblue}{rgb}{0, 0, 1.0}

\newcommand{\name}{\textsc{GraphPDHG}\xspace}

\title{Neural Algorithmic Reasoning for Graph Saddle-Point Problems}
\author{Samantha Chen\textsuperscript{1,2,5}\\ 
\texttt{schen41@oberlin.edu}
\and
Jesse He\textsuperscript{3}\\
\texttt{jeh@ucsd.edu}
\and Coleman Clougherty\textsuperscript{3}\\
\texttt{cclougherty@ucsd.edu}
\and Gal Mishne\textsuperscript{3}\\
\texttt{gmishne@ucsd.edu}
\and Chester Holtz\textsuperscript{4, 5}\\
\texttt{chholtz@ucsd.edu}}
\date{}

\begin{document}

\maketitle
\renewcommand{\thefootnote}{\arabic{footnote}}
\footnotetext[1]{Corresponding author}
\footnotetext[2]{Department of Computer Science, Oberlin College}
\footnotetext[3]{Department of Data Science, University of California, San Diego}
\footnotetext[4]{OneChronos}
\footnotetext[5]{Work done while at the University of California, San Diego}

\begin{abstract}
  Neural algorithmic reasoning, or aligning a neural network with an algorithmic paradigm, has emerged as an approach to solving polynomial-time-solvable and computationally harder combinatorial optimization problems. 
    We propose a new message-passing framework based on the Chambolle-Pock Primal--Dual Hybrid Gradient (PDHG) method called 
    \name for solving general graph saddle-point problems.
    Theoretically, we show that \name can efficiently solve a family of graph saddle-point problems by simulating PDHG. We also show that our network can learn an accelerated PDHG algorithm. 
    Experimentally, we support our results on accelerated PDHG by evaluating the performance of our model as a learned warm start for second-order optimization techniques (SSNAL). 
    We also show that alignment with PDHG leads to stronger size generalization than non-aligned graph neural network (GNN) baselines. Overall, we propose a novel architecture for solving a general family of optimization problems on graphs.
\end{abstract}

\section{Introduction}
Data-dependent methods to learn better and faster approximations to difficult combinatorial optimization problems have become increasingly popular and important for downstream applications like chip design, scheduling, matching, and related applications~\cite{bengio2021machine}. Neural algorithmic reasoning (NAR), where neural networks are designed to imitate algorithmic control flows~\citep{velickovic2020neural,velickovic2021nar,velickovic2022clrs}, has emerged as a promising technique for improving the accuracy of such models.
In the context of graph-structured problems, neural algorithmic reasoning has also been shown to improve out-of-distribution generalization, or the ability of a network to adapt to test instances that are sampled from a different distribution than the training instances~\citep{cappart2023combinatorial,yang2021gnniterative,yau2024optgnn,nerem2025graphneuralnetworksextrapolate}.
Neural algorithmic reasoning has been shown to be especially effective with respect to size generalization, when the size of the test input graph is larger than the size of the training graphs. 
By biasing the model toward  an algorithm, one hopes that the algorithm will learn a generalizable procedure that will deliver robust results across different problem instances and additionally, avoid expensive handcrafted feature engineering.

In this paper, we leverage neural algorithmic reasoning to design efficient neural models for solving graph-structured (non)convex-(non)concave saddle-point problems, which appear across statistics, signal processing and machine learning~\citep{chambolle2011first, tibshirani2005fused, chen2015convexclustering, hein2011cut}.
These problems include many standard graph-structured problems, including graph partitioning, convex clustering, and $p$-Wasserstein distance between graphs~\citep{hein2011cut,pmlr-v108-ye20a,jung19,chen2015convexclustering,chi2015splitting,peyré2020computationaloptimaltransport,robertson2025needresistanceequivalenceeffective}, and are most commonly adopted in problems where one must trade off between a reconstruction error and a regularity penalty under constraints. 
We are particularly motivated by the Primal--Dual Hybrid Gradient (PDHG) method~\citep{chambolle2011first,PockChambolle2011DiagonalPrecond}.
Classical PDHG is an \emph{iterative} general-purpose solver, but its practical performance depends strongly on certain instance-dependent choices (e.g., step sizes, diagonal preconditioning, extrapolation/momentum, stopping criteria)~\citep{applegate2021pdlp,pmlr-v84-mollenhoff18a,pmlr-v108-ye20a}. 
In the context of graph saddle point problems, when the reconstruction error and regularity penalty admit certain separable structure and are ``simple'' in the sense of computing their associated proximal operators, 
% In the context of the problem family of \Cref{eq:mainP}, when $G$ and $F$ admit certain separable structure and are ``simple'' in the sense of computing their associated proximal operators, 
PDHG is particularly efficient and amenable to message passing.
In contrast to other first-order or second-order methods, such as ADMM~\citep{chi2015splitting} or semismooth Newton methods~\cite{sun2021ssnal}, which require solving expensive global subproblems, PDHG can take advantage of the graph structure and only requires simple local message passing updates.

In this paper, we leverage this message passing structure to design a PDHG-inspired message passing neural network (MPNN). 
Importantly, our neural-PDHG architecture allows us to learn algorithmic parameters to the typical spectra/conditioning and feature statistics of the instances at hand.
This is especially valuable when one must solve many related graph problems under tight compute budgets, when good preconditioners are expensive to design, or when one has additional supervision.

To design an efficient neural model to solve these graph saddle-point problems, we are motivated by the Primal--Dual Hybrid Gradient (PDHG) method~\citep{chambolle2011first,PockChambolle2011DiagonalPrecond}.
Classical PDHG is an \emph{iterative} general-purpose solver, but its practical performance depends strongly on certain instance-dependent choices (e.g., step sizes, diagonal preconditioning, extrapolation/momentum, stopping criteria)~\citep{applegate2021pdlp,pmlr-v84-mollenhoff18a,pmlr-v108-ye20a}. 
In the context of graph saddle point problems, when the reconstruction error and regularity penalty admit certain separable structure and are ``simple'' in the sense of computing their associated proximal operators, 
% In the context of the problem family of \Cref{eq:mainP}, when $G$ and $F$ admit certain separable structure and are ``simple'' in the sense of computing their associated proximal operators, 
PDHG is particularly efficient and amenable to message passing.
In contrast to other first-order or second-order methods, such as ADMM~\citep{chi2015splitting} or semismooth Newton methods~\cite{sun2021ssnal}, which require solving expensive global subproblems, PDHG can take advantage of the graph structure and only requires simple local message passing updates.

In this paper, we leverage this message passing structure to design a PDHG-inspired message passing neural network (MPNN). 
Importantly, our neural-PDHG architecture allows us to learn algorithmic parameters to the typical spectra/conditioning and feature statistics of the instances at hand.
This is especially valuable when one must solve many related graph problems under tight compute budgets, when good preconditioners are expensive to design, or when one has additional supervision.

\noindent \textbf{Contributions.}
We introduce \name, a message passing GNN scheme inspired by PDHG which can be used to solve general graph saddle-point problems.
In contrast to earlier work~\citep{numeroso2023dual,he2025pdnar,li2024pdhgnet} which targets algorithmic \emph{execution} for discrete primal--dual schemes in neural algorithmic reasoning, we focus on \emph{continuous} primal--dual optimization for graph-structured composite objectives defined on point clouds and graphs. Our contributions include:
\begin{itemize}
    \item \textbf{A message passing PDHG network.} We introduce PDHG-inspired graph neural network for graph-based saddle-point problems with convex and nonconvex penalties, which we call \name. \name is designed to carefully exploit primal--dual saddle structure, while avoiding costly inner solves, and yielding simple, scalable iterations based on sparse graph operators and proximal updates. In Section~\ref{sec:method}, we introduce a node-edge message-passing layer that operates locally on a graph and is tailored to graph-regularized objectives $G(u)+F(Ku)$. We show in Section~\ref{sec:method_universality} that a single layer of \name can implement a single step of preconditioned PDHG. 
    % This MPNN implementation of PDHG allows us to interpret the learned parameters as algorithmic choices. We refer to this network as \name. 
    We apply our architecture to solving convex clustering.% and also introduce a novel PDHG algorithm for convex clustering. 
    
    \item \textbf{\name and learned acceleration.} In Section~\ref{sec:method_linearization}, we analyze the linearized primal--dual error dynamics induced by the layer and show that in a memoryless regime the update reduces to a preconditioned graph-Laplacian dynamics (Proposition~\ref{prop:memorylessdynamics}). Alternatively, with edge memory, the recurrence recovers heavy-ball acceleration~\citep{Polyak1964SomeMO} (Proposition~\ref{prop:memorydynamics}).
    
    \item \textbf{Empirical performance.} We evaluate our model performance with respect to the convex clustering problem. We show that our PDHG-based model produces high-quality warm starts for both PDHG and second-order optimization methods. We further demonstrate that \name generalizes better than GNN baselines to test graphs that are larger than those seen during training.
\end{itemize}
\paragraph{Related work.} 
We give more exposition on the existing literature in the appendix and restrict our attention here to the most relevant prior work. There is a large body of prior work in understanding how neural networks solve a variety of optimization problems from both supervised \cite{georgiev2024neural} and unsupervised approaches \cite{karalias2020erdos} to mixed neural-algorithmic pipelines \cite{kahng2024nn, mccarty2021nn}. The problems considered range from classical combinatorial optimization problems such as max-independent set and minimum dominating set \cite{loukas2019graph, sato2021random} to semidefinite programs~\cite{yau2024optgnn}. 
One of the earliest works to connect neural algorithmic reasoning with duality is \cite{numeroso2023dual}, where they develop dual algorithmic reasoning frameworks for solving min-cut/max-flow and shortest-path problems.
More recently, PDNAR \cite{he2025pdnar} and PDHG-net \cite{li2024pdhgnet} develop neural frameworks aligned with primal--dual algorithms, focusing on minimum hitting set and linear programming problems, respectively. 
In contrast to previous work, our model is explicitly designed to handle graph-based optimization problems. 
% Rather than relying only on the implicit structure of the optimization problem, we explicitly represent the learned primal variables as node features and learned dual variables as edge features. 
% Furthermore, the class of problems we consider is broader than linear programming and includes linear, nonlinear, and nonsmooth graph-structured saddle-point problems. 
We consider a broader class of problems which includes linear, nonlinear, and nonsmooth graph-structured saddle-point problems.
Finally, we differ from existing algorithmic reasoning works as \name is trained end-to-end to minimize the primal--dual gap in an \emph{unsupervised manner}, without the need for intermediate supervision or step-by-step algorithmic guidance. 

\section{Preliminaries} 

\paragraph{Graph neural networks.}
Graph neural networks (GNNs) are neural architectures for graph-structured data~\citep{scarselli2009graph}. Let $\mathcal{G}=(V,E,w)$ be a weighted graph with node features $X=\{x_i \in \mathbb{R}^d\}_{i\in V}$ and, optionally, edge features $e=\{e_{(i,j)}\}_{(i,j)\in E}$. Intuitively, a message-passing GNN computes new node features at each layer by aggregating information from its neighbors~\citep{gilmer2017neural,jegelka2022theory}. To be exact, in a generic layer, the feature for a node $i$ computed at each layer is
\begin{align}
    h_i^{(\ell+1)}
    &=
    f_{\mathrm{up}}^{(\ell)}
    \left(h_i^{(\ell)},f_{\mathrm{agg}}^{(\ell)}
    \left(
        \left\{
            h_j^{(\ell)}, e_{(i,j)}, w_{ij}
            : j\in\mathcal{N}(i)
        \right\}
    \right)\right),
\end{align}
where $\mathcal{N}(i)$ denotes the neighbors of $i$ and $h_i^{(\ell)}$ is the learned node feature. 
Since the aggregation and update maps are shared across vertices, a GNN layer can be applied to graphs with different sizes and topologies.
% Standard architectures such as graph convolutional networks (GCNs)~\citep{kipf2017semi} and graph attention networks (GATs)~\citep{velickovic2018graph} correspond to particular choices of these aggregation and update functions.

\paragraph{Graph saddle-point problems. }
We consider problems of the form
\begin{equation}\label{eq:main_P}
    \min_{u \in \mathbb{R}^{|V|\times d}}
    \left \{ G(u) + F(Ku) \right \},
\end{equation}
where the model is formulated on an undirected graph $\mathcal{G} = (V,E,w)$ with positive weights $w : E \to \mathbb{R}_{+}$. $G$ is commonly a function on vertices and $F$ is a function on edges. In the canonical graph-structured setting, $K$ is a weighted incidence operator. Fix an arbitrary orientation of each edge $e=(i,j)\in E$ and define
$
    (Ku)_e = \sqrt{w_e} (u_i - u_j).
$
Writing $\nabla$ for the (signed) incidence matrix of $\mathcal{G}$, one can compactly express
$
    K = \operatorname{diag}(\sqrt{w})\nabla,\:\: \forall e=(i,j) \in E
$
so that $Ku$ collects weighted differences across edges.

An equivalent way to express \eqref{eq:main_P} is via its saddle-point (primal–dual) formulation. The Fenchel duality framework gives the equivalent saddle problem
\begin{equation}\label{eq:saddle_P}
    \min_{u \in \mathbb{R}^{|V| \times d}} \max_{p \in \mathbb{R}^{|E| \times d}}
    \left\{G(u) + \langle Ku, p\rangle - F^{*}(p)\right\},
\end{equation}
where $F^*$ denotes the convex conjugate of $F$, and $\langle Ku, p\rangle = \sum_{e\in E} \langle (Ku)_e, p_e\rangle$.
Under standard qualification conditions, \eqref{eq:saddle_P} has the same optimal value as \eqref{eq:main_P}, and saddle points $(u^\star,p^\star)$ satisfy the KKT system associated with the original problem.
\begin{remark}[Primal--dual gap]
    Given the objective $f(u) = G(u) + F(Ku)$, the dual objective is $d(p) = -G^*(-K^\top p) - F^*(p)$, where $G^*$ and $F^*$ are convex conjugates. The primal--dual gap is
    \begin{equation}\label{eq:pdg}
        \begin{aligned}
            &\text{Gap}(u,p) = f(u) - d(p) = G(u) + F(Ku) + G^*(-K^\top p) + F^*(p) \geq 0   
        \end{aligned}
    \end{equation}
    Under the strong duality condition, $\text{Gap}(u^*, p^*) = 0$ iff $u^*$ and $p^*$ are the optimal primal and dual variables respectively. Otherwise, the gap is strictly positive and weak duality holds. 
\end{remark}

\paragraph{Convex clustering.}
Classical convex clustering (and similarly Network Lasso) of points $x_1,\dots,x_n \in \mathbb{R}^d$ can be written as a special case of \eqref{eq:main_P}. We illustrate this in the Appendix in Table~\ref{tab:notation}. Let the graph $\mathcal{G}=(V,E,w)$ have vertices $V = \{1,\dots,n\}$, edge weights $w_e > 0$ for $e=(i,j)\in E$. Define
\begin{align*}
    G(u) &= \frac{1}{2}\sum_{i\in V} \|u_i - x_i\|_2^2, \quad
    F(z) = \sum_{e\in E} \lambda_e \|z_e\|_2,
\end{align*}
where $\lambda_e := \lambda \sqrt{w_e}$. Then \eqref{eq:main_P} becomes
\begin{equation}\label{eq:convexclustering}
    \min_{u \in \mathbb{R}^{|V|\times d}} \frac12\sum_{i\in V}\|u_i - x_i\|_2^2
           + \lambda \sum_{(i,j)\in E} w_{ij} \|u_i - u_j\|_2,
\end{equation}
See the Appendix for more details.
% \begin{figure}[t]
%     \centering
%     \includegraphics[width=0.995\linewidth]{figures/convex-clustering.png}
%     \caption{Example of convex clustering on a $k$-nearest neighbor graph for hierarchical Gaussian blobs. We show the solutions to convex clustering ~\eqref{eq:convexclustering} as the fusion weight $\lambda$ increases. \sam{Move to appendix}}
%     \label{fig:convexclustering}
% \end{figure}

\paragraph{Primal--Dual Hybrid Gradient (PDHG).}
% The saddle-point formulation \eqref{eq:saddle_P} is naturally compatible with a node-edge representation of graph optimization. The primal variable $u_i\in\mathbb{R}^d$ lives on vertex $i$, while the dual variable $p_e\in\mathbb{R}^d$ lives on edge $e$. The incidence operator $K$ maps node states to edge differences, while its transpose $K^\top$ maps edge variables back to vertices by aggregating the signed incident-edge contributions.
The preconditioned Primal--Dual Hybrid Gradient (PDHG) method is a popular framework for solving the problem in \eqref{eq:saddle_P}. 
Generic PDHG iterations alternate between these two local operations. 
First, the current node variables are compared across each edge, and the corresponding edge dual variable is updated by a proximal or projection step. 
Second, each node aggregates the updated dual variables on its incident edges and uses this aggregate to update its own primal variable.
Thus, PDHG has the same computational pattern as node--edge message passing: node states send differences to edges, and edge states send corrective messages back to nodes.
More concretely, a general formulation of the PDHG iterations is given as follows:
\begin{align}
    &u^{(k+1)} = \arg\min_{u\in \mathbb{R}^{|V| \times d}} G(u) + \langle p^{(k)}, Ku \rangle + \frac{s}{2} \lVert u - u ^{(k)} \rVert_S^2 \label{eq:pdhg-u}\\
    &p^{(k+1)} = \arg\min_{p\in \mathbb{R}^{|E| \times d}} F^{*}(p) - \langle K(2u^{(k+1)} - u^{(k)}),\:\: p \rangle + \frac{t}{2} \lVert p - p^{(k)}\rVert_T^2 \label{eq:pdhg-p}  
\end{align}
Here, $S \in \mathbb{R}^{|V| \times |V|}$ and $T \in \mathbb{R}^{|E| \times |E|}$ are symmetric positive definite matrices, $s,t > 0$ and the norm $\lVert \cdot \rVert_A$ for some matrix $A$ of a vector $u$ is defined by the quadratic $\lVert u \rVert^2_A = u^\top A u$. 
For a given $S$ and $T$, the convergence of PDHG is guaranteed if the (inverse) step sizes $s$ and $t$ satisfy 
$st > \lVert T^{-1/2} KS^{-1/2}\rVert^2.$
One key feature of PDHG is that the updates described in \eqref{eq:pdhg-u} and \eqref{eq:pdhg-p} become fully local under certain separability assumptions on $G$ and $F$ and their associated proximal maps. As we will demonstrate, this observation is the main algorithmic reason for using PDHG as the template for \name. 
While classical PDHG fixes the proximal maps, step sizes, preconditioners, and extrapolation rule, in \name, we allow certain parts of the update to be learned, so that the network can adapt its effective preconditioning and acceleration to a family of graph topologies and feature statistics of a problem distribution.

% \sam{Notice that (7) and (8) together are purely local operations (along the edges of the graph) and distributed on the nodes. Therefore, it is amenable to the message passing formulation described above. This observation is the main algorithmic reason for using PDHG as the template for \name. While classical PDHG fixes the proximal maps, step sizes, preconditioners, and extrapolation rule, in \name, we allow certain parts of the update to be learned, so that the network can adapt its effective preconditioning and acceleration to a family of graph topologies and feature statistics of a problem distribution.}
% \chester{As we previously mentioned, one key feature of PDHG is that these updates become fully local under certain separability assumptions on $G$ and $F$ and their associated proximal maps. As we will demonstrate, this observation is the main algorithmic reason for using PDHG as the template for \name. While classical PDHG fixes the proximal maps, step sizes, preconditioners, and extrapolation rule, in \name, we allow certain parts of the update to be learned, so that the network can adapt its effective preconditioning and acceleration to a family of graph topologies and feature statistics of a problem distribution.}

\begin{theorem}[\cite{PockChambolle2011DiagonalPrecond}]\label{thm:pdhg_convergence}
    Assume $G$ and $F$ are proper, lower semicontinuous, convex, and that \eqref{eq:saddle_P} admits a saddle point
    $(u^\star,p^\star)$. Let $(u^{(k)},p^{(k)})$ be generated by \eqref{eq:pdhg-u} and \eqref{eq:pdhg-p} such that
    $$
        st > \bigl\|T^{-1/2}KS^{-1/2}\bigr\|_2^2.
    $$
    Let $\bar u^{(N)}:=\frac1N\sum_{k=1}^N u^{(k)}$ and
    $\bar p^{(N)}:=\frac1N\sum_{k=1}^N p^{(k)}$. Then for all $N\ge 1$,
    \begin{align*}
        &\mathcal L(\bar u^{(N)},p^\star)-\mathcal L(u^\star,\bar p^{(N)}) \le
        \frac{1}{N}\left(\frac{s}{2}\|u^{(0)}-u^\star\|_{S}^{2}
        +\frac{t}{2}\|p^{(0)}-p^\star\|_{T}^{2}\right),
    \end{align*}
    where $\mathcal L(u,p):=G(u)+\langle Ku,p\rangle-F^*(p)$, i.e., the primal--dual gap converges as $\mathcal O(1/N)$.
\end{theorem}
%The following corollary also demonstrates that the Gap serves as an upper bound on the solution error. 
The above theorem also motivates our unsupervised training objective. In particular, in Corollary~\ref{cor:strong_convexity_gap_to_distance} in the Appendix, we show that when $G$ is strongly convex, the primal--dual gap bounds the distance from the primal iterate to the unique primal optimizer. Thus, even without access to ground-truth solutions or intermediate PDHG iterates, minimizing the gap controls the iterate error. This is one important advantage of \name over recent work in NAR, which generally require intermediate layer hints or algorithmic supervision.

\section{Message passing PDHG} \label{sec:method}

\iffalse
\begin{table}[tb]
\caption{Examples of separable convex and nonconvex penalties. We derive associated proximal and $\mathsf{NORMALIZE}$ maps in the Appendix.}\label{tab:edgepenalties}
    \begin{center}
        \small
        \setlength{\tabcolsep}{6pt}
        \renewcommand{\arraystretch}{1.15}
        \begin{tabular}{@{}p{}p{}@{}}
            \toprule
            \textbf{Name} & \textbf{Penalty $F_e(z)$} \\
            \midrule
            \multicolumn{2}{@{}l}{\textbf{Convex penalties}}\\
            \addlinespace[2pt]
            Group TV ($\ell_2$) & $\lambda_e\|z\|_2$ \\
            Anisotropic TV ($\ell_1$) & $\lambda_e\|z\|_1$ \\
            Max-norm TV ($\ell_\infty$) & $\lambda_e\|z\|_\infty$ \\
            \addlinespace[3pt]
            \multicolumn{2}{@{}l}{\textbf{Nonconvex penalties}}\\
            \addlinespace[2pt]
            $\ell_0$ & $\lambda_e\mathbf{1}\{z\neq 0\}$ \\
            Clipped $\ell_1$\citep{zhang2010clippedlasso} & $\lambda_e\sum_{j=1}^d \min(|z_j|,t)$ \\
            SCAD \citep{fan2001scad} & $\lambda_e\mathrm{SCAD}_{a}(\|z\|_2)$ \\
            MCP \citep{zhang2010mcp} & $\lambda_e\mathrm{MCP}_{\gamma}(\|z\|_2)$ \\
            \bottomrule
        \end{tabular}
    \end{center}
\end{table}
\fi
%\begin{wraptable}[15]{r}{0.5\textwidth}

%
%
The PDHG formulation (\eqref{eq:saddle_P}) accommodates a broad range of edge penalties. In the separable setting, where
$
    F(z)=\sum_{e\in E} F_e(z_e),
$
the choice of $F_e$ specifies the coupling prior across edges. In the Appendix, we list a few examples of separable convex and nonconvex edge penalties with simple proximal operators.

We now present our framework using a GNN to simulate the general primal--dual approximation algorithm, representing the primal--dual as node and edge features, respectively.

%\subsection{\name architecture} \label{sec:method_architecture}
%\begin{figure*}[!t]
%    \centering
%    \includegraphics[width=0.65\linewidth]{figures/PDHG-Graph.png}
%    \caption{Overview of GNN feature propagation and the \name layer architecture}
%    \label{fig:placeholder}
% \end{figure*}

We adopt the encoder-processor-decoder architecture from~\cite{hamrick2018relational} implemented in the neural algorithmic reasoning pipeline~\cite{velickovic2021nar}.
In this framework, the \emph{processor} is taken to be \name, operating on the latent space learned by an \emph{encoder}. The \emph{decoder} reconstructs the final representation of the solution from the latent space. The framework is trained end-to-end to minimize the \emph{primal--dual gap}. 
In this framework, the processor network is meant to execute a general numerical algorithm
%, or do most of the work of the network, 
while the encoder and decoder are simple layers used to transform the data to and from the latent representations required by the processor module to natural outputs.
Note that although only the processor part of the pipeline uses our \name architecture, we still refer to the entire encoder-processor-decoder with a \name processor as \name.

\paragraph{Encoder. } The encoder architecture of \name consists of two MLP encoders, $f_n$ and $f_e$, which encode the node and edge inputs. These encoders transform the feature inputs into a high-dimensional latent space for the processor:
$$
    x_0 = f_n(x),\quad e^{(0)} = f_e(e)
$$
\paragraph{Processor. } A single layer of the processor $\Phi_\theta$ is defined by 
$$
    (h^{(\ell+1)}, e^{(\ell + 1)}) = \Phi_{\theta}(h^{(\ell)}, e^{(\ell)}; \mathcal{G}, X_0).
$$
We define a message passing layer for PDHG as follows. 

Let $\theta = (A_1, A_2, B_1, B_2, B_3)$ be a collection of learnable linear maps and $\varphi$ a differentiable nonlinearity. Define
\begin{align*}
    &f_{\text{up}}^{\text{edge}}\left(e^{(\ell)}_{(i,j)}\right) = A_1e^{(\ell)}_{(i,j)}, \quad f^{\text{edge}}_{\text{agg}}\left(h_i^{(\ell)} - h_j^{(\ell)}\right) = A_2\left(\sqrt{w_{ij}}(h_i^{(\ell)} - h_j^{(\ell)})\right), \\
    &f_{\text{up}}^{\text{node}}\left(h_j^{(\ell)}, \sum_{i\in N(j)}e_{(i,j)}^{(\ell+1)}\right) = \varphi \left(B_1h_j^{(\ell)} + B_2 (x_0)_j + B_3 \sum_{i\in N(j)}e_{(i,j)}^{(\ell+1)}\right ), \\
    &\mathsf{NORMALIZE}(e) = \text{prox}_{\sigma F^*}(e).
\end{align*}
One step of PDHG message passing is defined as
\begin{align*}
    e_{(i,j)}^{(\ell + 1)} &= \mathsf{NORMALIZE}(f_{\text{up}}^{\text{edge}}(e_{(i, j)}^{\ell}) + f_{\text{agg}}^{\text{edge}}(h_i^{(\ell)} - h_j^{(\ell)}))\\
    h_{j}^{(\ell + 1)} &= f_{\text{up}}^{\text{node}}(h_j^{(\ell)}, \sum_{i \in N(j)} e_{(i,j)}^{(\ell + 1)}).
\end{align*}
The $L$-layer processor is the composition of $L$ PDHG message passing layers.
$
    \mathcal{N}_\theta^{L} := \Phi_{\theta_{L}} \circ \ldots \circ \Phi_{\theta_1}.
$
The layers of the processor can be interpreted as a learnable version of one PDHG step. The matrix $A_1$ controls how much edge-dual information is retained from the previous layer. The matrix $A_2$ controls how node differences are transformed before updating the edge state. The matrices $B_1$, $B_2$, and $B_3$ combine the previous node state, the original input features, and the aggregated edge feedback. In classical PDHG, these transformations are fixed scalars (e.g., step sizes) or simple preconditioners; in \name, they are learned from a distribution of training graph instances.

\paragraph{Decoder. } As with the encoder, the architecture implements two decoders for edges and nodes: $\widehat{f}_n$, $\widehat{f}_e$ which map from the latent space to the data space.
%\begin{remark}
%    The framework is implemented in a recurrent (iterative) manner: after the initial encoding of node/edge inputs into latent states, the processor block $\Phi_\theta$ is applied repeatedly for $T$ steps, with shared parameters across steps. At iteration $\ell$, the processor updates edge states (dual-like variables) from current node states, then updates node states (primal-like variables) using aggregated incident edge messages, yielding  $(h(T), e (T)) = \mathcal{N}_\theta(h^{(0)}, e^{(0)};\mathcal{G}, X_0)$. The decoder is applied only after the final (or averaged) iterate.
%\end{remark}
We show in the following theorem that this architecture is universal. If $G$ decomposes over nodes, $F$ decomposes over edges, and the associated proximal maps are separable, then for any desired accuracy, there exists a choice of parameters with depth bounded by a quantity depending on the spectral norm of the graph Laplacian such that the network achieves an arbitrarily small primal--dual gap. The proof is deferred to the appendix.
\begin{theorem}[Universality]\label{thm:universality}
    Let $\epsilon > 0$ be a target primal--dual gap. There exists a parameter configuration $\theta^*$ and a depth $L$ satisfying
    $$
        %L \leq \frac{C \cdot \lVert K \rVert^2}{\epsilon}
        L \leq \mathcal{O}\left(\frac{\lVert K \rVert^2}{\epsilon}\right)
    $$
    such that the network achieves an ergodic primal--dual gap of at most $\epsilon$. Here, $\lVert K\rVert$ is the spectral norm of the weighted incidence matrix. 
\end{theorem}

\subsection{Feedforward dynamics of \name} \label{sec:method_linearization}
We have shown that \name is expressive enough to implement PDHG. In this section, we will show that our proposed architecture is also capable of learning an accelerated optimization algorithm as compared to vanilla PDHG. Our results here imply that for graphs that exhibit rapid eigenvalue decay, \name can learn an accelerated algorithm.
% (1) Here, we will show that our proposed architecture can also learn an accelerated optimization algorithm (2) First, at each layer, our PDHG-net will eliminate the error associated with one eigenvector (3) the depth of the GNN is independent of the size of the input graph.
We consider the linearized feedforward dynamics of PDHG-inspired message passing. Specifically, we characterize these dynamics under the following set of assumptions:
\begin{assumption} \label{def:linear_regime}
    The layer is operating in the Linear Regime if the following two conditions hold:
    \begin{enumerate}
    \item The activation function $\varphi_{\text{node}}(\cdot)$ is the identity. 
    \item For the projection $\text{Proj}_{\mathcal{C}}(P)$ onto the $\ell_2$-ball constraints, there exists a partition of edges $E = E_{\text{active}} \cup E_{\text{inactive}}$. In the linear regime, the projection acts as the identity on $E_{\text{inactive}}$.
    \end{enumerate}
\end{assumption}
%
%\sam{Maybe we don't need to say our focus on the linear regime is "justified". }
%While the global dynamics of \name are nonlinear due to the projection operator and nonlinear node activations, we claim our focus on the \emph{linear regime} is justified. 
In the classical analysis of nonlinear proximal methods, the iterates $(u^{(k)}, p^{(k)})$ converge in two stages. In the first stage, the set of ``active'' constraints that lie on the boundary of the feasible set $\mathcal{C}$ stabilizes~\cite{Liang03062018}. Once this active set is identified, the projection operator $\text{Proj}_{\mathcal{C}}$ simplifies locally to an affine operator~\cite{HareLewis2004IdentifyingActiveConstraints}. Namely, it acts as the identity for inactive variables and a projection onto the tangent space for active variables.
Within this linear regime, the linearization of our framework is characterized by the following updates:
\begin{remark}\label{rem:linear_regime}
    Let $\mathcal{I} = \{e \in E : \|p_e\| < \lambda_e\}$. In the linear regime, the message passing updates are:
    \begin{align*}
        e^{(\ell+1)} &= e^{(\ell)}A_1^\top +  K U^{(\ell)}A_2^\top \\
        U^{(\ell+1)} &= U^{(\ell)}B_1 ^\top + (K^\top e^{(\ell+1)})B_3^\top + C
    \end{align*}
    where $A_1, A_2$ and $B_1, B_3$ are the network parameters, and $C$ is some constant. Substituting the edge update into the node update recovers the following recursion:
    \begin{align*}
        U^{(\ell+1)}
        &=U^{(\ell)}B_1 ^\top + (K^\top e^{(\ell)})A_1^\top B_3^\top + (K^\top K)U^{(\ell)} A_2^\top B_3^\top  + C,
    \end{align*}
    where $K^\top K$ is the graph Laplacian.
\end{remark}
Analyzing the recursion yields a pair of results. The first result demonstrates that a certain parameterization implies fast convergence on expander-type graphs that exhibit rapid eigenvalue decay. In particular, the result below says that each layer can eliminate a source of error associated with one eigen-direction of the graph Laplacian. 

\begin{proposition}[Memoryless dynamics]~\label{prop:memorylessdynamics}
    Let $\Delta = K^\top K \in \mathbb{R}^{|V| \times |V|}$ be the graph Laplacian restricted to the inactive edge set defined in the Linear Regime. Let $\{(\lambda_k, \psi_k)\}_{k=1}^{|V|}$ be the orthonormal eigenpairs of $\Delta$. Define the error matrix $\mathcal{E}^{(\ell)} = U^{(\ell)} - U^*$. 
    
    For any target eigenvalue $\lambda_k \neq 0$, there exists a parameterization of the GNN layer such that the error at the next layer, $\mathcal{E}^{(\ell+1)}$, satisfies
    $
        \psi_k^\top \mathcal{E}^{(\ell+1)} = 0.
    $
    That is, the projection of the error onto the $k$-th eigenspace is eliminated for all feature dimensions.
\end{proposition}
While this result implies the algorithm learned by \name can rapidly converge on graphs with fast eigenvalue decay, the implied theoretical convergence rate is still $O(\kappa)$. Below, we demonstrate that the network has the capacity to learn faster, accelerated algorithms that converge in $O(\sqrt{\kappa})$. In particular, retaining edge memory yields a second-order recurrence which bears a similarity to classical accelerated first-order methods.
\begin{proposition}[Dynamics with memory]\label{prop:memorydynamics}
    Let the GNN operate in the linear regime with edge memory $A_1 = \gamma I_d$ for some $\gamma \in (0,1)$. The primal error dynamics satisfy the second-order recurrence of the Polyak Heavy Ball method~\cite{Polyak1964SomeMO}:
    $$
        \mathcal{E}^{(\ell+1)} = \mathcal{E}^{(\ell)} + \gamma (\mathcal{E}^{(\ell)} - \mathcal{E}^{(\ell-1)}) - \Delta \mathcal{E}^{(\ell)} M,
    $$
    where $M$ is a learnable effective step size matrix.
\end{proposition}

\begin{corollary}\label{cor:heavyball}
    Consider the simplification $M = -A_2^\top B_3^\top = \eta I_d$ for $\eta > 0$. Then, the worst-case error decays at the accelerated rate 
    $
        \max_k \|c_k^{(\ell)}\|_2 \le C \left(\frac{\sqrt{\kappa}-1}{\sqrt{\kappa}+1}\right)^\ell 
    $
    for some constant $C$.
\end{corollary}
The results above imply convergence that is \emph{independent of the size of the graph} for families of graphs with bounded condition number $\kappa$. This is a unique property of \name and is a notable shortcoming of prior NAR works  ~\cite{loukas2019graph}. We improve the rate in \Cref{cor:heavyball_joint_kappa}. 
\begin{corollary}\label{cor:layerbound}
    Let the GNN operate in the linear regime and let the primal error dynamics satisfy the second-order recurrence of the Polyak Heavy Ball method as in \Cref{prop:memorydynamics}. Assume that at the first layer, the solution found by the neural network incurs some error $c$ for each vertex. For graphs with bounded condition number $\kappa$, there is a parameter configuration for a depth $\mathcal{O}(\kappa)$ GNN such that $\mathcal{E} < \mathcal{O}(\kappa, 1/\varepsilon)$.
\end{corollary}

% \sam{This means that there is a bounded depth GNN which can solve this class of saddle-point problems, given graphs with bounded condition number. Contrast this with the results of \citep{loukas2019graph}, where the depth of the GNN is required to scale with the size of the input graph. }
% \sam{Expander graphs have bounded $\kappa$. }
% In the previous pair of results, the role of the parameterized matrix $M = -A_2^\top B_3^\top$ is opaque. Here we demonstrate that it plays an important role in conditioning the features of each layer.

\subsection{PDHG for convex clustering}
\label{subsec:pdhg-convex-clustering}
Here, we use convex clustering (\eqref{eq:convexclustering}) as a case study for how our \name framework can be directly applied. As previously noted, convex clustering is an important continuous relaxation of the clustering objective where one must balance between a reconstruction objective and a fusion penalty. This makes it a natural choice for PDHG to solve. We derive the explicit message passing form of PDHG for convex clustering. This message passing form can then be directly used in our \name framework to derive a neural version of PDHG for convex clustering.

For the choice of $F$ implied by the convex clustering objective in \eqref{eq:convexclustering}, the convex conjugate is
\begin{align*}
    &F^*(p)
    = \sum_{e\in E} \mathbb{I}_{\{\|p_e\|_2 \le \lambda_e\}}(p_e)
    = \mathbb{I}_{\mathcal{C}}(p), \quad \mathcal{C} := \left\{p : \|p_e\|_2 \le \lambda_e \ \forall e\in E\right\}.
\end{align*}
For the quadratic term, $G$, the convex conjugate, used in \eqref{eq:pdg},  is
$$
    G^*(q)
    = \sum_{i\in V} \left( \langle q_i, x_i\rangle
      + \tfrac12 \|q_i\|_2^2 \right)
    = \langle q, x\rangle + \tfrac12 \|q\|_2^2.
$$
We now specialize the preconditioned PDHG iterations to this problem. Let $S \in \mathbb{R}^{|V|\times |V|}$ and $T \in \mathbb{R}^{|E|\times |E|}$ be symmetric positive definite, and let $s,t>0$ be the (inverse) step sizes. Recall the generic PDHG updates \eqref{eq:pdhg-u} and \eqref{eq:pdhg-p}.
To obtain local message passing updates, we choose node- and edge-wise preconditioners $S = \operatorname{diag}(S_i)_{i\in V}$ and $T = \operatorname{diag}(T_e)_{e\in E}$,
with $S_i \in \mathbb{R}^{d\times d}$, $T_e\in \mathbb{R}^{d\times d}$ symmetric positive definite blocks. Then the proximal operator (\eqref{eq:pdhg-p}) decouples over edges and simplifies, for each $e \in E$:
\begin{equation}
    p_e^{(k+1)}
    = \operatorname{Proj}_{\{\|p\|_2 \le \lambda_e\}}
      \left(
          p_e^{(k)} + \frac{1}{t}T_e^{-1} (Ku^\mathrm{extr})_e
      \right),\label{eq:cc-edge-update}
\end{equation}
where $u^\mathrm{extr}$ denotes the extrapolated primal variable (e.g.,\ $2u^{(k)} - u^{(k-1)}$), and $(Ku^\mathrm{extr})_e = \sqrt{w_e}(u_i^\mathrm{extr} - u_j^\mathrm{extr})$ for $e=(i,j)$.

Similarly, \eqref{eq:pdhg-u} decouples over vertices. The first-order optimality condition for each $i\in V$ yields
\begin{equation}
    (I_d + s S_i) u_i^{(k+1)}
    = s S_i u_i^{(k)} + x_i
      - (K^\top p^{(k+1)})_i, \label{eq:cc-node-update}
\end{equation}
where
$
    (K^\top p)_i
    = \sum_{e=(i,j)\in E} \sqrt{w_e} p_e.
$
Thus each primal update is a local linear transform of the previous iterate, the data $x_i$, and the aggregated dual variables on incident edges.

\begin{remark}
    A particularly simple choice is
    $
        S = I,\quad T = I,
        s = \tfrac{1}{\sigma}, t = \tfrac{1}{\tau},
    $
    for step sizes $\sigma,\tau>0$ satisfying $\tau\sigma \lVert K\rVert^2 < 1$. In this case, \eqref{eq:cc-edge-update} and \eqref{eq:cc-node-update} simplify to the explicit message passing
    \begin{align}
        p_{ij}^{(k+1)}
        &= \operatorname{Proj}_{\{\|p\|_2 \le \lambda_e\}}
           \left(
               p_{ij}^{(k)} + \tau \sqrt{w_{ij}} \left(u_i^\mathrm{extr} - u_j^\mathrm{extr}\right)
           \right),\\
        u_i^{(k+1)}
        &= \frac{1}{1+\sigma}
           \left(
               u_i^{(k)} + \sigma x_i
               - \sigma \sum_{j : (i,j)\in E} \sqrt{w_{ij}} p_{ij}^{(k+1)}
           \right).
    \end{align}
\end{remark}

\section{Experimental Results}
\label{sec:experiments}
In our experiments, we focus on two main questions: Can a PDHG-aligned model provide effective learned preconditioners or warm starts for generic optimization algorithms?  Does alignment with Primal--Dual Hybrid Gradient (PDHG) promote size generalization? 
Warm-starting is especially important for graph-structured optimization problems, as classical solvers such as PDHG and semismooth Newton methods exhibit runtime that depends nontrivially on the quality of the initial primal and dual variables. 
We use \name{} to amortize high-quality initializations over a distribution of related graphs by producing primal--dual variables that are already near the relevant saddle point or have approximately identified the correct fused-edge structure. 
In our experiments, we demonstrate \name{}'s utility for convex clustering (see \Cref{subsec:pdhg-convex-clustering}). See our appendix for all experimental details including experiments on other saddle-point problems. 
% Our appendix includes all experimental details and includes additional experiments such as ablations, details relating to computational cost, and the empirical convergence rate vs. the accelerated theoretical rate described in \Cref{sec:method_linearization}. 

% Then, on out-of-sample data we allow existing robust and mature classical algorithms to reach certified high-accuracy solutions in few iterations.
To evaluate \name on multi-scale geometric graphs, we use a synthetic hierarchical Gaussian cluster dataset, and three image datasets. The hierarchical Gaussian cluster dataset consists of graphs whose node features are sampled from a recursive Gaussian hierarchy. For the image datasets, we consider $k$-nearest neighbor graphs generated using features from MNIST~\citep{deng2012mnist}, Fashion-MNIST~\citep{xiao2017fashion}, and CIFAR-10~\citep{cifar}.
% We generate graphs for these datasets using pretrained autoencoders as feature extractors and then constructing $k$-nearest neighbor graphs from subsampled features.
% See appendix for all details regarding datasets.

\begin{wrapfigure}[16]{r}{0.45\textwidth}
    \centering
    \vspace{-20pt}
    \includegraphics[width=0.98\linewidth]{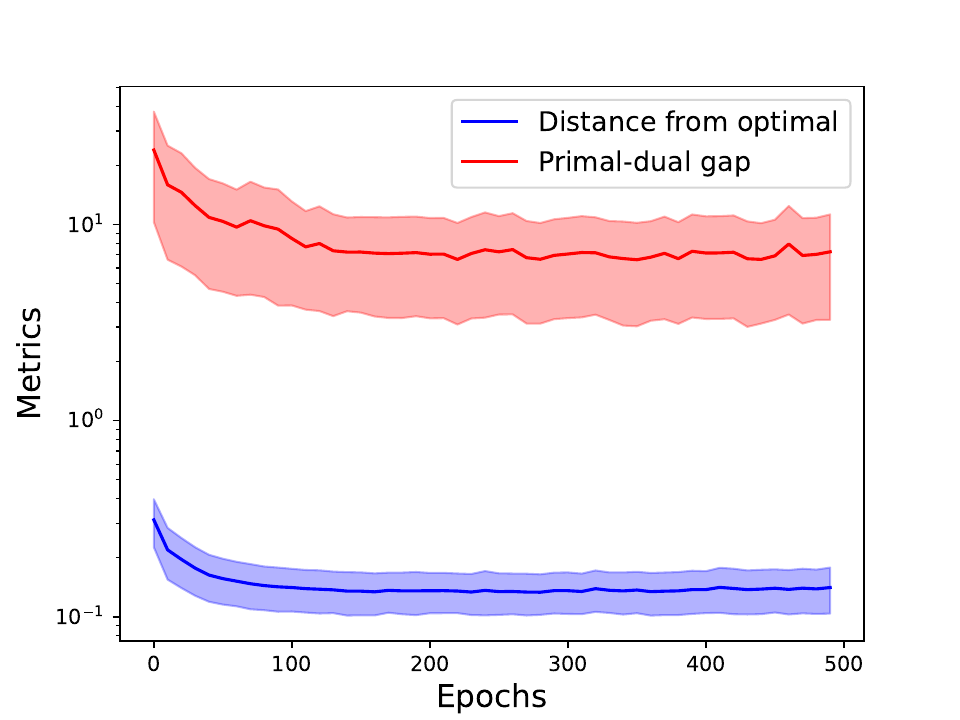}
    \caption{Tracking the primal--dual gap and the distance between the predicted and optimal primal variables throughout training. The iterate error shrinks as the gap shrinks. }
    \vspace{-30pt}
    \label{fig:pdg-vs-distance}
\end{wrapfigure}
We train \name using the primal--dual gap (\eqref{eq:pdg}).
By \Cref{cor:strong_convexity_gap_to_distance}, the primal--dual gap is an upper bound on the distance to the optimal solution, so a small gap implies that the predicted primal variables are close to optimal. Minimizing the primal--dual gap allows us to train the model in an unsupervised manner without requiring exact solutions or step-by-step algorithmic supervision. 
We observe that this theoretical relationship is reflected empirically. See \Cref{fig:pdg-vs-distance}. We also note that our theoretical results assume \emph{ergodic} iterates. In our experiments we use the last layer readouts.

% often solved repeatedly across related instances, graph sizes, or regularization parameters. 
% \sam{Chester: can you say a few words here about why warmstarting is important/lead to faster/better algorithms? }
% \chester{
% Warm-starting is important because graph-structured optimization problems are  often solved repeatedly across related instances, graph sizes, or regularization parameters. Classical solvers such as PDHG and semismooth Newton methods exhibit runtime that depends nontrivialy on the quality of the initial primal and dual variables. For example, this dependence is explicit in the PDHG bound in Theorem~\ref{thm:pdhg_convergence}, where the constant multiplying the $\mathcal{O}(1/N)$ rate contains the initial distances $\lVert u^{(0)}-u^\star \rVert_S^2$ and $\lVert p^{(0)}-p^\star \rVert_T^2$. Our goal with \name{} is to amortize high-quality initializations over a distribution of related graphs by producing primal--dual variables that are already near the relevant saddle point or have approximately identified the correct fused-edge structure. Then, on out of sample data we allow existing robust and mature classical algorithms to reach certified high-accuracy solutions in few iterations.
% }

We compare against several standard message-passing GNN baselines: GCN~\citep{kipf2017semi}, GAT~\citep{velickovic2018graph}, GIN~\citep{xu2019powerful}, GraphSAGE~\citep{hamilton2017inductive}, and EdgeConv~\citep{wang2019dynamic}. For each baseline, we keep the same encoder--processor--decoder template and replace only the \name{} processor with the corresponding GNN processor. All models were trained using the same primal--dual gap objective and experimental protocol.

\begin{figure}[b]
    \centering
    \includegraphics[width=0.31\linewidth]{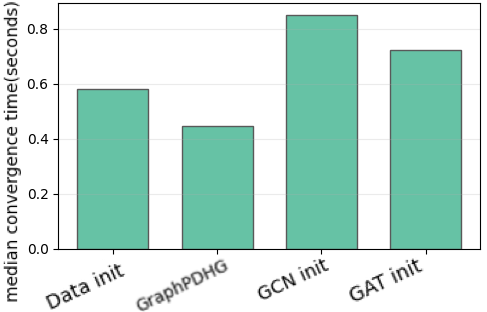}
    \raisebox{0.6em}{%
        \includegraphics[width=0.3\linewidth]{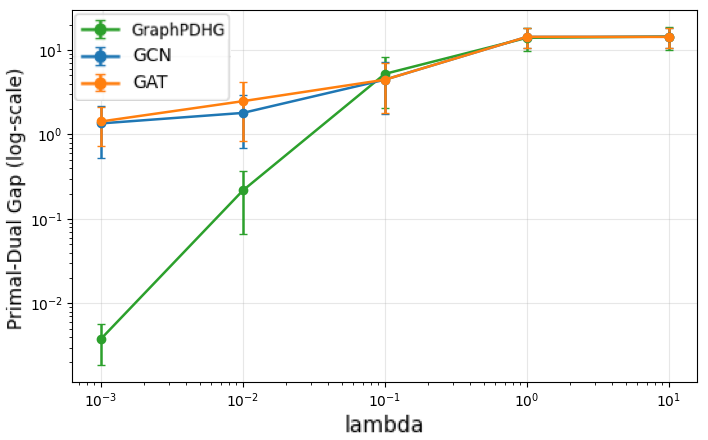}
    } 
    \raisebox{0.8em}{%
        \includegraphics[width=0.3\linewidth]{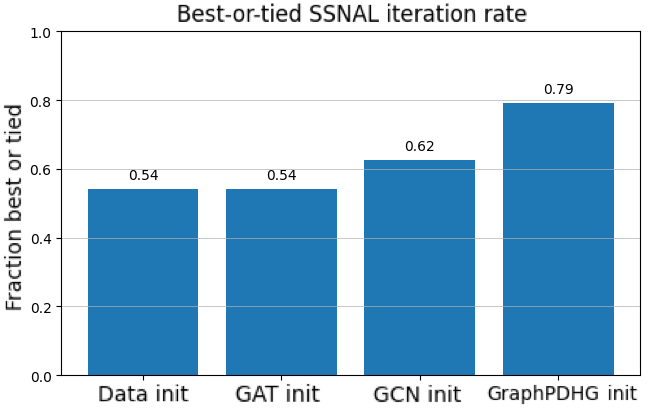}
    }
    \caption{Warm-starting SSNAL for the hierarchical Gaussian dataset. \textbf{Left:} Time (seconds) to recover a primal--dual gap $\leq 10^{-6}$. \textbf{Middle:} The primal--dual gap of different network initializations as $\lambda$ increases. \textbf{Right:} For each initialization method, the fraction of validation graphs for which SSNAL reaches a primal--dual gap of $\leq 10^{-6}$ at least as fast as the best-performing method.
}
    \label{fig:runtime-hierarchical-gaussians}
    \vspace{-0.2cm}
\end{figure}

\paragraph{Initializing a High-Accuracy Solver with \name{}.}
We evaluate \name as a learned warm start for the semismooth Newton augmented Lagrangian method (SSNAL), a state-of-the-art second-order solver for convex clustering~\citep{sun2021ssnal}. Rather than replacing SSNAL, \name provides initial primal--dual variables that are passed to the classical solver. We measure whether this initialization reduces the iterations required to reach a target primal--dual gap, indicating that the network has learned useful problem structure such as the approximate fused or active dual-edge variables. We note that although \name is aligned with PDHG, SSNAL is a robust high-accuracy solver whose performance benefits from an initialization that is close to the optimal primal--dual solution. Thus, we focus our experiments on SSNAL.

The results are summarized for the hierarchical Gaussian dataset in \Cref{fig:runtime-hierarchical-gaussians}. \name provides a better warm-start than both simply initializing from the data or initializing from the baseline GNN outputs. The performance converges at higher $\lambda$ as all methods snap to the trivial solution (i.e., the grand mean in the center).
The results for the image datasets are shown in \Cref{tab:first-hit-times} where we record the number of iterations needed to reach different target primal--dual gap values. We show \name initialization reduces the number of iterations SSNAL needs to converge. 

\begin{table}[t]
\centering
\scriptsize
\setlength{\tabcolsep}{4.5pt}
\caption{First-hit times for different initialization methods. Entries report mean $\pm$ standard deviation over runs. The best result is bold. Median hit times and ablations are reported in the Appendix.}
\label{tab:first-hit-times}
    \begin{tabular}{llcccccc}
    \toprule
    & & \multicolumn{3}{c}{$\lambda = 2.5$} 
      & \multicolumn{3}{c}{$\lambda = 10$} \\
    \cmidrule(lr){3-5} \cmidrule(lr){6-8}
    \textbf{Dataset} & \textbf{Method}
    & $10^{-4}$ & $10^{-6}$ & $10^{-8}$
    & $10^{-4}$ & $10^{-6}$ & $10^{-8}$ \\
    \midrule
    \textbf{MNIST}
    & Data  & $2.0 \pm 0.0$  & $11.6 \pm 4.8$ & $11.6 \pm 4.8$ 
            & $2.0 \pm 0.0$ & $6.4 \pm 2.6$ & $6.4 \pm 2.6$ \\
    & GIN   & $2.0 \pm 0.0$  & $12.0 \pm 5.8$ & $12.0 \pm 5.8$
            & $2.2 \pm 0.3$ & $8.7 \pm 4.8$ & $8.7 \pm 4.8$ \\
    & GCN   & $2.0 \pm 0.0$  & $13.4 \pm 6.8$ & $13.4 \pm 6.8$
            & $2.2 \pm 0.2$ & $8.5 \pm 4.7$ & $8.5 \pm 4.7$ \\
    & GAT   & $2.0 \pm 0.0$  & $7.0 \pm 3.4$  & $7.0 \pm 3.4$
            & $2.4 \pm 0.2$ & $9.1 \pm 5.1$ & $9.1 \pm 5.1$ \\
    & GraphSAGE & $2.0 \pm 0.0$ & $8.8 \pm 4.5$ & $8.8 \pm 4.5$
                & $2.2 \pm 0.2$ & $7.4 \pm 3.8$ & $7.4 \pm 3.8$ \\
    & EdgeConv  & $2.0 \pm 0.0$ & $6.4 \pm 3.0$ & $6.4 \pm 3.0$
                & $2.1 \pm 0.2$ & $6.4 \pm 2.9$ & $6.4 \pm 2.9$ \\
    & \name & $2.0 \pm 0.0$  & $\mathbf{4.7 \pm 2.1}$  & $\mathbf{4.7 \pm 2.1}$
            & $\mathbf{2.0 \pm 0.0}$ & $\mathbf{5.6 \pm 2.3}$ & $\mathbf{5.6 \pm 2.3}$ \\
    \midrule
    \textbf{Fashion-MNIST}
    & Data  & $5.4 \pm 0.4$ & $12.3 \pm 4.9$ & $12.3 \pm 4.9$
            & $5.2 \pm 0.4$ & $11.4 \pm 4.4$ & $11.4 \pm 4.4$ \\
    & GIN   & $5.5 \pm 0.6$ & $14.5 \pm 6.8$ & $14.5 \pm 6.8$
            & $5.3 \pm 0.7$ & $14.0 \pm 6.8$ & $14.0 \pm 6.8$ \\
    & GCN   & $5.4 \pm 0.6$ & $13.6 \pm 6.5$ & $13.6 \pm 6.5$
            & $5.2 \pm 0.6$ & $13.1 \pm 6.3$ & $13.1 \pm 6.3$ \\
    & GAT   & $6.7	 \pm 1.2$ & $17.8 \pm 8.1$ & $17.8 \pm 8.1$
            & $6.4 \pm 1.3$ & $21.3 \pm 9.9$ & $21.3 \pm 9.9$ \\
    & GraphSAGE & $5.3 \pm 0.5$ & $12.0 \pm 5.2$ & $12.0 \pm 5.2$
                & $5.1 \pm 0.5$ & $11.8 \pm 5.0$ & $11.8 \pm 5.0$ \\
    & EdgeConv  & $5.2 \pm 0.5$ & $10.8 \pm 4.7$ & $10.8 \pm 4.7$
                & $5.0 \pm 0.5$ & $10.6 \pm 4.8$ & $10.6 \pm 4.8$ \\
    & \name & $\mathbf{5.1 \pm 0.4}$ & $\mathbf{9.9 \pm 4.3}$ & $\mathbf{9.9 \pm 4.3}$
            & $\mathbf{4.9 \pm 0.4}$ & $\mathbf{9.4 \pm 4.2}$ & $\mathbf{9.4 \pm 4.2}$ \\
    \midrule
    \textbf{CIFAR-10}
    & Data  & $\mathbf{2.3 \pm 0.4}$ & $5.8 \pm 2.4$ & $5.8 \pm 2.4$
            & $2.2 \pm 0.4$ & $5.4 \pm 2.1$ & $5.4 \pm 2.1$ \\
    & GIN   & $2.5 \pm 0.6$ & $6.8 \pm 4.0$ & $6.8 \pm 4.0$
            & $2.4 \pm 0.6$ & $6.5 \pm 3.7$ & $6.5 \pm 3.7$ \\
    & GCN   & $2.5 \pm 0.6$ & $6.3 \pm 3.7$ & $6.3 \pm 3.7$
            & $2.4 \pm 0.6$ & $6.1 \pm 3.4$ & $6.1 \pm 3.4$ \\
    & GAT   & $2.5 \pm 0.6$ & $12.5 \pm 6.6$ & $12.5 \pm 6.6$
            & $2.4 \pm 0.6$ & $15.0 \pm 7.3$ & $15.0 \pm 7.3$ \\
    & GraphSAGE & $2.4 \pm 0.5$ & $6.0 \pm 3.0$ & $6.0 \pm 3.0$
                & $2.3 \pm 0.5$ & $5.9 \pm 2.9$ & $5.9 \pm 2.9$ \\
    & EdgeConv  & $\mathbf{2.3 \pm 0.4}$ & $5.5 \pm 2.3$ & $5.5 \pm 2.3$
                & $2.2 \pm 0.4$ & $5.2 \pm 2.2$ & $5.2 \pm 2.2$ \\
    & \name & $\mathbf{2.3 \pm 0.4}$ & $\mathbf{5.1 \pm 1.9}$ & $\mathbf{5.1 \pm 1.9}$
            & $\mathbf{2.1 \pm 0.4}$ & $\mathbf{4.8 \pm 1.8}$ & $\mathbf{4.8 \pm 1.8}$ \\
    \bottomrule
    \end{tabular}
    \vspace{-0.1cm}
\end{table}
\begin{figure}[t]
    \centering

    \centering
    \includegraphics[width=0.245\linewidth]{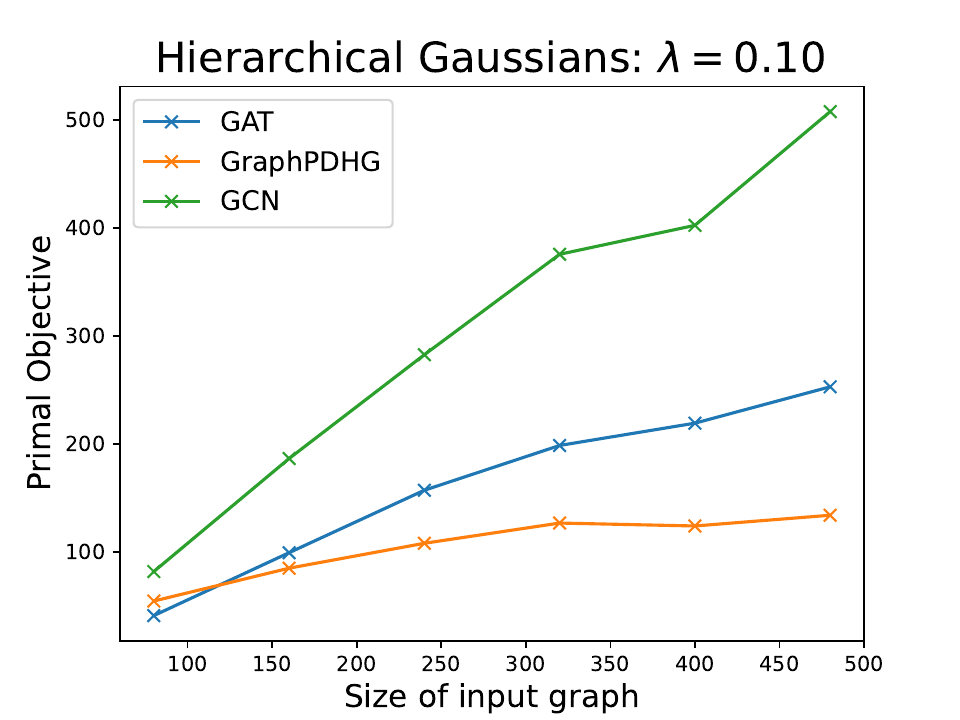}
    \includegraphics[width=0.245\linewidth]{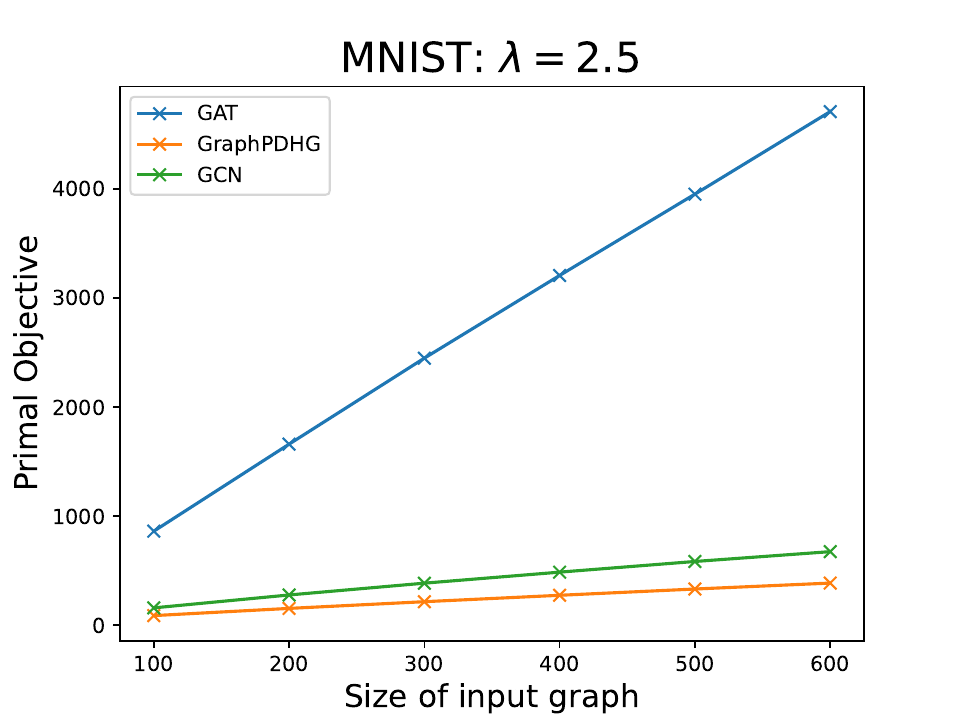}
    \includegraphics[width=0.245\linewidth]{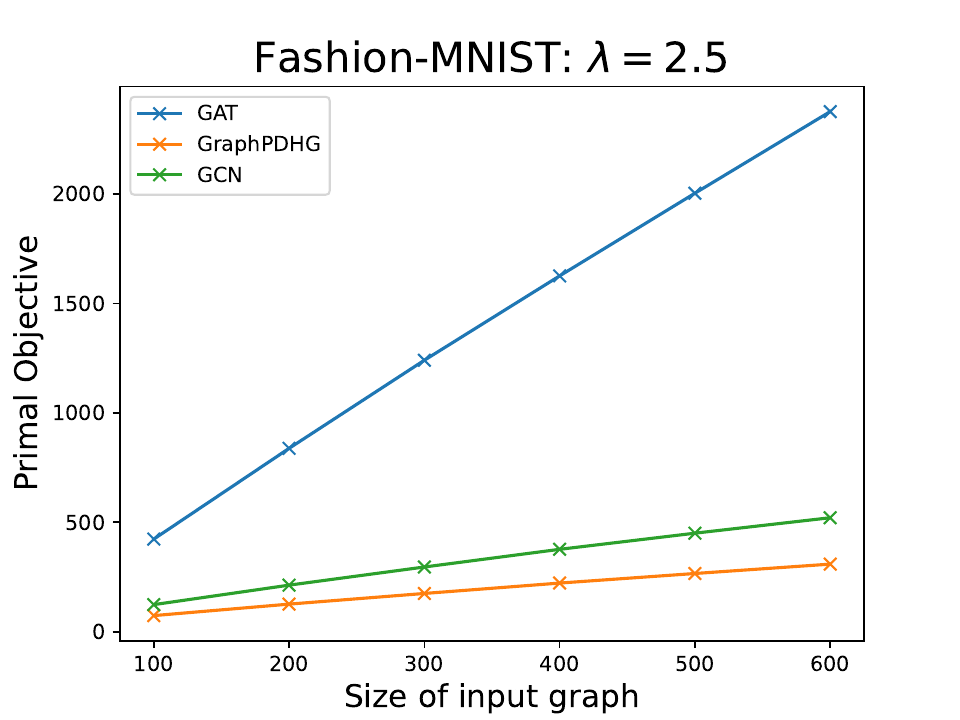}
    \includegraphics[width=0.245\linewidth]{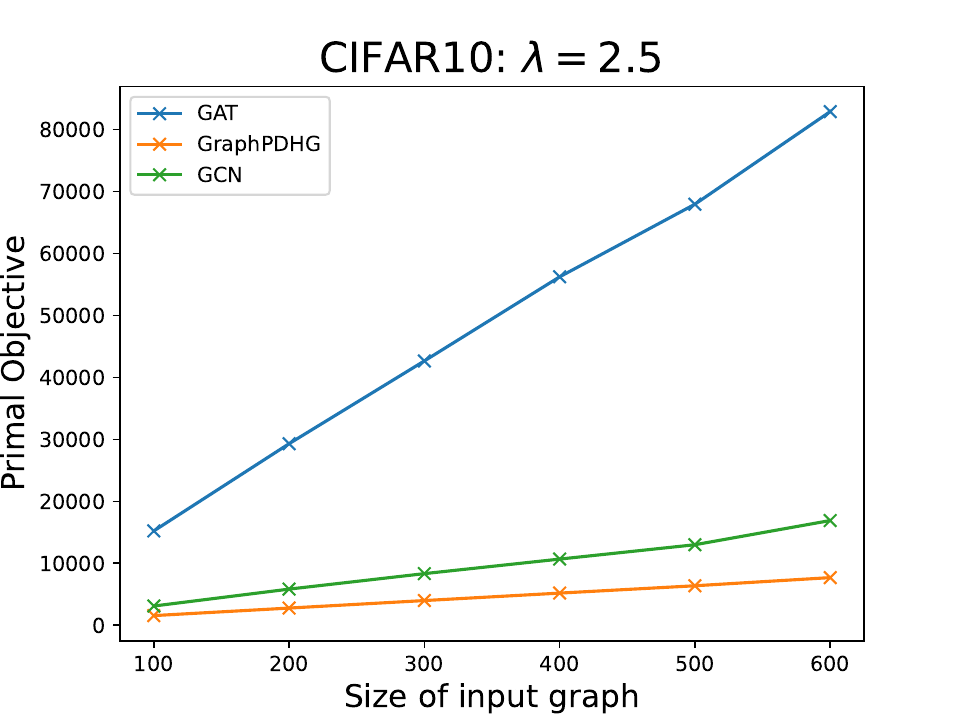}

    \caption{
    % \gal{fonts are too small in the plots. } 
    Size generalization for each dataset. We track the primal objective as the size of the test graphs increases for all datasets. Note that we test at a fixed $\lambda$ for all graphs. 
    }
    \label{fig:size-generalization}
\end{figure}
\paragraph{Size generalization.} We evaluate the size generalization of all neural network methods by measuring the primal objective value of the predicted initialization from each network on increasingly large graphs. The results are summarized in \Cref{fig:size-generalization} across all datasets. For all datasets, \name{} achieves substantially lower primal objective values than the GAT and GCN baselines as the graph size increases. This suggests that \name{} generalizes more effectively to larger problem instances. We attribute this behavior to the alignment of \name with PDHG. By explicitly mirroring the structure of PDHG, \name{} appears to learn a general rule for initialization that transfers more robustly than GAT or GCN across graph sizes. 

% \noindent{\textbf{Ablations.}} In \Cref{tab:init-primal-dual-gap} in the appendix we report the primal–dual gap of each initialization. \name achieves the lowest gap across all datasets, while both fixed Unrolled PDHG and \name without the edge projection/normalization step perform significantly worse, suggesting that the gains come from combining primal–dual message passing with the dual-feasibility projection.

\section{Conclusion}

In this paper, we introduced \name, a primal--dual message-passing framework for graph-structured saddle-point problems. The central observation is that PDHG has the same computational structure as node-edge message passing. Theoretically, we showed that the proposed architecture can exactly represent preconditioned PDHG updates and that its linearized dynamics admit interpretations in terms of the Laplacian spectrum and acceleration. Empirically, we demonstrated the effectiveness of \name on convex clustering problems, where the model improves size generalization and provides useful warm starts for classical solvers. These results suggest that learned primal--dual message passing can serve not only as a standalone approximation method, but also as a mechanism to accelerate classical optimization algorithms.

\subsection*{AI use statement}

In this work, we used generative AI tools for editing and writing tasks. Additionally, while the experimental pipeline was written by the authors without the use of AI tools, we used generative AI tools running experiments and writing code to create figures. We have reviewed all AI-assisted work. All LLM-generated code was verified and tested for correctness by two authors. We take responsibility for the final content of this work,
including text, claims or artifacts produced with the aid of generative AI.



\bibliography{main}
\bibliographystyle{plain}

\clearpage

\appendix
\section{Appendix}

\subsection{Convex clustering}

% \noindent \textbf{Convex Clustering and Network Lasso. }
%

Many familiar models arise as instances of \eqref{eq:main_P}. For example, choosing the function
$
    F(z) = \sum_{e\in E} \lambda_e \|z_e\|_2
$
(or its square) for edge-wise weights $\lambda_e = \lambda \sqrt{w_e}$ yields a penalty $F(Ku) = \lambda\sum_e  w_e \|u_i - u_j\|_2$ used in semi-supervised learning, network lasso, and convex clustering. 
Similarly, choosing $F(z) = \sum_{e \in E} \lambda_e\lvert z_e\rvert$ (or other $\ell_p$ norms) yields a graph total-variation penalty that appears in graph cuts~\cite{robertson2025robustgraphbasedsemisupervisedlearning} and graph TV denoising. 

In this paper, we focus on the particular case of convex clustering.
Convex clustering~\cite{chi2025convexclustering} is a continuous relaxation of clustering in which each data point $x_i\in\mathbb{R}^d$ is assigned a learnable representative (i.e., centroid), $u_i\in\mathbb{R}^d$. 
The objective balances two goals: each centroid should remain close to its original data point, while centroids connected by large graph weights should be encouraged to coincide via a nonsmooth term in the objective. 
The regularization parameter $\lambda$ controls this tradeoff. For small $\lambda$, the centroids remain close to the input points; as $\lambda$ increases, neighboring centroids fuse into common representatives, producing a hierarchy of coarser clusterings. This solution path is illustrated in \Cref{fig:convexclustering}.

More concretely, classical convex clustering (and similarly Network Lasso) of points $x_1,\dots,x_n \in \mathbb{R}^d$ can be written as a special case of \eqref{eq:main_P}. We illustrate this in \Cref{tab:notation}. Let the graph $\mathcal{G}=(V,E,w)$ have vertices $V = \{1,\dots,n\}$, edge weights $w_e > 0$ for $e=(i,j)\in E$. Define
\begin{align*}
    G(u) &= \frac{1}{2}\sum_{i\in V} \|u_i - x_i\|_2^2, \quad
    F(z) = \sum_{e\in E} \lambda_e \|z_e\|_2,
\end{align*}
where $\lambda_e := \lambda \sqrt{w_e}$. Then \eqref{eq:main_P} becomes
\begin{equation}%\label{eq:convexclustering}
    \min_{u \in \mathbb{R}^{|V|\times d}} \frac12\sum_{i\in V}\|u_i - x_i\|_2^2
           + \lambda \sum_{(i,j)\in E} w_{ij} \|u_i - u_j\|_2,
\end{equation}
which is the standard convex clustering objective (with $u_i$ the cluster centroids). Here, $G$ can be thought of as the reconstruction error and $F$ as the regularity penalty. Generally, \eqref{eq:convexclustering} is solved using variable substitution or splitting methods~\citep{chi2015splitting} or dual projected gradient methods~\citep{pi2021dualconvexclustering}. Recently, semismooth second-order methods have also been proposed~\citep{sun2021ssnal}. In the main text of our paper, we derive the PDHG updates for this problem and demonstrate how these updates are computationally amenable to local message passing.
\begin{figure}[t]
    \centering
    \includegraphics[width=0.8\linewidth]{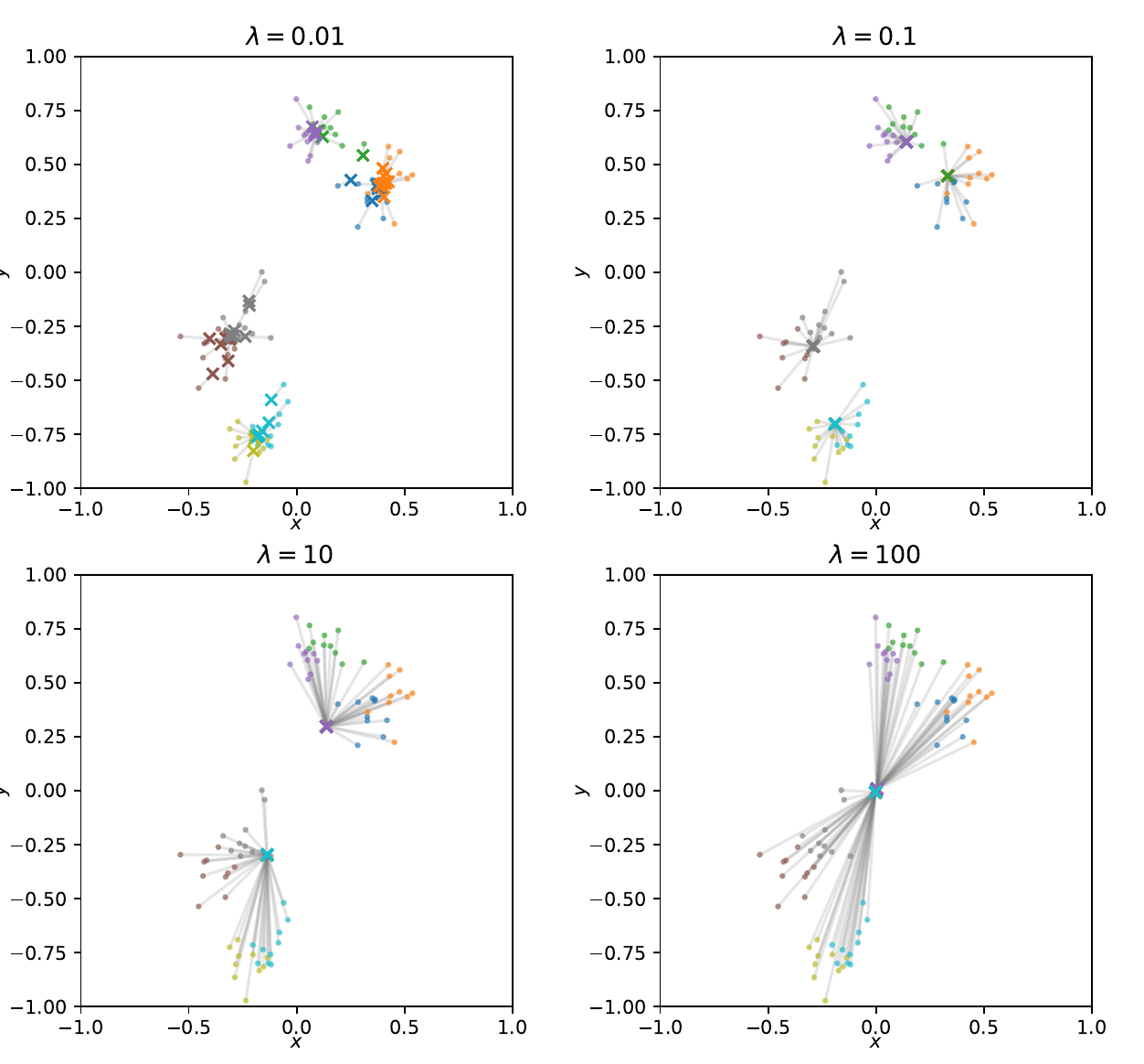}
    \caption{Solutions for convex clustering~(\eqref{eq:convexclustering}) as $\lambda$ increases on a $k$-NN graph for hierarchical Gaussians.}
    \label{fig:convexclustering}
    %\vspace{-0.7cm}
    % \vspace{-0.8cm}
\end{figure}
\subsection{Related Work}
\label{sec:related_work}

\paragraph{Optimization problems on graphs and primal--dual algorithms.}
Graph-structured optimization problems date back 20 years. The fused lasso is based on total-variation (TV) distances on a chain graph~\citep{tibshirani2005fused}, while the network lasso extends TV-style coupling to general graphs for prediction and clustering~\citep{hallac2015network}. Closely related are the family of
convex clustering problems on point clouds~\citep{chi2015splitting, chen2015convexclustering, sun2021convexclustering}. These problems introduce continuous representatives, ``centroids" for each sample and promote fusion along a similarity graph. Related applications include optimal transport on graphs, including linear, nonlinear, and constrained variants~\cite{peyré2020computationaloptimaltransport}\textemdash and two-player matrix games on graphs~\cite{Bapat1997}.
On the learning side, TV-based semi-supervised learning on graphs has been studied both statistically and algorithmically~\citep{jung18, jung19}, and TV-inspired GNN layers have been used as inductive biases for clustering and pooling~\citep{hansen2023total}.

A complementary body of work studies nonconvex penalties as surrogates for $\ell_0$-type objectives, including Smoothly Clipped Absolute Deviation (SCAD)~\citep{fan2001scad} and the Minimax Concave Penalty (MCP)~\citep{zhang2010mcp}.
Algorithmically, coordinate-descent and proximal methods for these penalties have been studied in depth~\citep{breheny2011nonconvex}, and theoretical analyses highlight settings where nonconvex regularization can achieve stronger support recovery guarantees~\citep{loh2015regularized, loh2017nonconvex}.
More recently, convex--nonconvex constructions have been proposed for graph TV to encourage edge preservation while maintaining global tractability~\citep{heng2025convexnonconvex}.

Many of these objectives can be written as composite problems of the form $\min_u G(u)+F(Ku)$ with $K$ a (weighted) incidence operator. As we will see, one can also consider an equivalent \emph{saddle point} representation of the form $\min_u \max_y$.
The Primal--Dual Hybrid Gradient (PDHG) method~\citep{chambolle2011first} is a standard first-order approach for such saddle-point formulations, and its practical efficiency often hinges on conditioning and step-size selection~\cite{applegate2021pdlp}. Notably,
combinatorial preconditioning schemes exploit graph structure to accelerate proximal algorithms for graph-TV-type problems~\citep{pmlr-v84-mollenhoff18a, pmlr-v108-ye20a}.

\paragraph{Loop unrolling and neural algorithmic reasoning.}
A large literature on \emph{learned unrolling} turns iterative optimization methods into trainable architectures by learning parameters such as step sizes, thresholds, or preconditioners~\citep{gregor2010lista,zhang2018istanet,xiediffadmm2019, yang2016deepadmm,chen2022l2o,sambharya2023qpinit}.
On graphs, GNNs have also been designed to mimic classical iterative algorithms and inherit stability/contraction properties~\citep{yang2021gnniterative}, while neural algorithmic reasoning (NAR) aims to learn execution traces of algorithms and generalize to larger inputs~\citep{velickovic2020neural, velickovic2021nar, velickovic2022clrs, nerem2025graphneuralnetworksextrapolate}.
Beyond generic NAR, duality-based formulations have been proposed to jointly learn paired algorithms (e.g., max-flow/min-cut)~\citep{numeroso2023dual}, and a recent primal--dual NAR benchmark aligns primal/dual variables with a bipartite message-passing representation~\citep{he2025pdnar}.

Most closely related to this work are learned primal--dual and PDHG-inspired networks.
PDNAR~\citep{he2025pdnar} and PDHG-Net~\citep{li2024pdhgnet} both derive unrolled primal--dual iterations for GNNs and show that sufficiently expressive parameterizations can recover classical PDHG updates. More concretely, PDNAR investigates simulation of discrete primal–dual approximation algorithms using a bipartite primal/dual representation and additional supervision obtained by running classical PDHG. PDHG-Net is closer algorithmically, but it is specialized to linear programs. In particular, PDHG-Net requires constraints specific to the LP formulations considered in~\citep{li2024pdhgnet}. The problems we consider do not naturally fit within that framework, so PDHG-Net cannot be used as a processor network in our architecture. In contrast, our framework targets generic saddle-point problems and learns directly from the primal–dual gap, without additional supervision. The mechanism is a graph-local message-passing realization of preconditioned PDHG designed for \emph{generic nonlinear and nonsmooth} objectives of the form $G(u)+F(Ku)$. We also provide theoretical results that interpret the learned updates and size generalization through the lens of preconditioning and acceleration.
In a broader optimization-on-graphs context, OptGNN demonstrates that message passing can mimic relaxation-based solvers for Max-CSPs~\citep{yau2024optgnn}, and surveys catalog many combinatorial-optimization uses of GNNs~\citep{cappart2023combinatorial}.
Our contribution is complementary: we target generic graph-structured composite convex objectives and give an explicit correspondence between standard message-passing layers and preconditioned PDHG as well as connect learned dynamics to classical acceleration ideas such as heavy-ball momentum~\citep{Polyak1964SomeMO}.

\begin{table}[t]
\centering
\caption{Notation: convex clustering as an instance of \eqref{eq:main_P}.}
\label{tab:notation}
    \begin{tabular}{lll}
        \toprule
        \textbf{Object} & \textbf{Graph interpretation} & \textbf{Convex clustering interpretation} \\
        \midrule
        $u_i$ & node variable & centroid for point $x_i$ \\
        $(Ku)_e$ & weighted edge difference & discrepancy between neighboring centroids \\
        $p_e$ & edge dual variable & fusion forcing term \\
        $K^\top p$ & aggregate incident edge feedback & net force on centroid $u_i$ from incident edges \\
        $F(Ku)$ & graph regularizer & penalty encouraging fused centroids \\
        $F^*(p)$ & dual constraint/penalty & feasible range of edge tensions \\
        $G(u)$ & node objective & fidelity / data reconstruction term \\
        \bottomrule
    \end{tabular}
\end{table}
\subsection{Other penalties and their proximal maps ($\mathsf{NORMALIZE}$)}
\begin{table}
\centering
\caption{Examples of separable convex and nonconvex penalties.
}
\label{tab:edgepenalties}
\small
\setlength{\tabcolsep}{4pt}
\renewcommand{\arraystretch}{1.1}
\begin{tabular}{@{}p{0.24\textwidth}p{0.24\textwidth}@{}}
    \toprule
    \textbf{Name} & \textbf{Penalty $F_e(z)$} \\
    \midrule
    \multicolumn{2}{@{}l}{\textbf{Convex penalties}}\\
    \addlinespace[2pt]
    Group TV ($\ell_2$) & $\lambda_e\|z\|_2$ \\
    Anisotropic TV ($\ell_1$) & $\lambda_e\|z\|_1$ \\
    Max-norm TV ($\ell_\infty$) & $\lambda_e\|z\|_\infty$ \\
    \addlinespace[3pt]
    \multicolumn{2}{@{}l}{\textbf{Nonconvex penalties}}\\
    \addlinespace[2pt]
    $\ell_0$ & $\lambda_e\mathbf{1}\{z\neq 0\}$ \\
    Clipped $\ell_1$\citep{zhang2010clippedlasso} & $\lambda_e\sum_{j=1}^d \min(|z_j|,t)$ \\
    SCAD \citep{fan2001scad} & $\lambda_e\mathrm{SCAD}_{a}(\|z\|_2)$ \\
    MCP \citep{zhang2010mcp} & $\lambda_e\mathrm{MCP}_{\gamma}(\|z\|_2)$ \\
    \bottomrule
\end{tabular}
\vspace{-0.5cm}
\end{table}
Because the edge penalty is separable,
$$
    F(z)=\sum_{e\in E}F_e(z_e),
$$
both $F^*$ and its proximal map are also separable across edges.

Therefore, the normalization operator
$$
    \mathsf{NORMALIZE}(y)=\operatorname{prox}_{\sigma F^*}(y)
$$
acts edgewise, and it suffices to compute the map for a single edge.

In this section we suppress the edge index and write $F$ for a generic penalty on $\mathbb R^{d}$. See \Cref{tab:edgepenalties} for several examples of convex and non-convex penalties. We derive associated proximal below.

\subsubsection{Convex penalties}
For any norm $\lVert \cdot \rVert$ with dual norm $\lVert \cdot \rVert_*$, consider
$$
    F(z) = \lambda \lVert z \rVert.
$$
Its convex conjugate is the indicator onto the dual norm ball~\cite{parikh2014proximal}:
$$
    F^*(p) = \mathbb{I}_{\{\lVert p \rVert_* \leq \lambda\}} (p)
$$
Consequently,
$$
    \operatorname{prox}_{\sigma F^*}(y) =
    \arg\min_{p}\frac12\|p-y\|_2^2+\sigma \mathbb I_{\{\|p\|_*\le \lambda\}}(p) = \text{Proj}_{\{\|p\|_*\le \lambda\}}(y),
$$
which is simply the Euclidean projection onto the dual norm ball and is independent of $\sigma$.

\paragraph{$\ell_2$ penalty}
Let $F(z)=\lambda \lVert z\rVert_2$. Then, $F^*(p)=\mathbb I_{\{\|p\|_2\le \lambda\}}(p)$, and
$$
    \mathsf{NORMALIZE}(y) = \text{Proj}_{\{\|p\|_2\le \lambda\}}(y) = 
    \begin{cases}
        y, & \|y\|_2\le \lambda,\\[3pt]
        \dfrac{\lambda}{\|y\|_2}y, & \|y\|_2>\lambda.
    \end{cases}
$$

\paragraph{$\ell_1$ penalty}
Let $F(z)=\lambda \lVert z\rVert_1$. Then, $F^*(p)=\mathbb I_{\{\|p\|_\infty\le \lambda\}}(p)$, and the $\mathsf{NORMALIZE}$ map is the coordinatewise clipping
$$
    [\mathsf{NORMALIZE}(y)]_j = \text{Proj}_{\{\|p\|_\infty\le \lambda\}}(y) =  \operatorname{sign}(y_j)\min(|y_j|,\lambda),
\:\: j=1,\dots,d.
$$

\paragraph{$\ell_\infty$ penalty}
Let $F(z)=\lambda \lVert z\rVert_\infty$. Then, $F^*(p)=\mathbb I_{\{\|p\|_1\le \lambda\}}(p)$, and
$$
    \mathsf{NORMALIZE}(y) = \text{Proj}_{\{\|p\|_1\le \lambda\}}(y).
$$

\subsubsection{Nonconvex penalties}

The previous rules rely on convex duality. For nonconvex penalties, the standard convex-conjugate PDHG formulation no longer applies directly. Here, we define the $\mathsf{NORMALIZE}$ map using the direct proximal map
$$
    \mathsf{NORMALIZE}^+(y):=\operatorname{prox}_{\tau F}(y).
$$
Below we describe two simple examples.

\paragraph{$\ell_0$ penalty}

Consider
$$
    F(z)=\lambda\mathbf 1\{z\neq 0\}.
$$
The proximal problem is 
$$
    \min_{u}\frac{1}{2\tau}\lVert u-y\rVert_2^2+\lambda\mathbf 1\{u\neq 0\}.
$$
The solution set  consists of two relevant candidates: $u = 0$ and $u = y$. Their objective values are
$$
    \frac{1}{2\tau}\|y\|_2^2 \text{ and } \lambda,
$$
Therefore, the proximal map is the thresholding rule:
$$
    \operatorname{prox}_{\tau \lambda \mathbf 1\{\cdot\neq 0\}}(y)
    =
    \begin{cases}
        0, & \lVert y\rVert_2\leq \sqrt{2\tau\lambda},\\[3pt]
        y, & \lVert y\rVert_2>\sqrt{2\tau\lambda}.
    \end{cases}
$$

\paragraph{Clipped $\ell_1$}
Let
$$
    F(z)=\lambda\sum_{j=1}^d \min(|z_j|,t).
$$
Since $F$ is separable, the proximal map acts coordinatewise. The resulting scalar prox is a soft-thresholding rule up to the clipping level, and equals the identity beyond that level:
$$
    [\mathsf{NORMALIZE}^+(y)]_j=
    \begin{cases}
        0, & |y_j|\le \tau\lambda,\\[4pt]
        \operatorname{sign}(y_j)\bigl(|y_j|-\tau\lambda\bigr),
        & \tau\lambda<|y_j|\le t+\dfrac{\tau\lambda}{2},\\[8pt]
        y_j, & |y_j|> t+\dfrac{\tau\lambda}{2}.
    \end{cases}
$$

\paragraph{MCP}
Let
$$
    F(z)=\mathrm{MCP}_{\lambda,\gamma}(\|z\|_2),
$$
where
$$
    \mathrm{MCP}_{\lambda,\gamma}(r)=
    \begin{cases}
        \lambda r-\dfrac{r^2}{2\gamma}, & 0\le r\le \gamma\lambda,\\[8pt]
        \dfrac{\gamma\lambda^2}{2}, & r>\gamma\lambda.
    \end{cases}
$$
Because $F$ is radial, $\operatorname{prox}_{\tau F}(y)$ is collinear with $y$, so the problem reduces to a one-dimensional minimization in the radius. This gives the group firm-thresholding map
$$
    \mathsf{NORMALIZE}^+(y)=
    \begin{cases}
        0, & \|y\|_2\le \tau\lambda,\\[8pt]
        \dfrac{\gamma}{\gamma-\tau}
        \left(1-\dfrac{\tau\lambda}{\|y\|_2}\right)y,
        & \tau\lambda<\|y\|_2\le \gamma\lambda,\\[12pt]
        y, & \|y\|_2>\gamma\lambda,
    \end{cases}
    \qquad (\tau<\gamma).
$$

\subsection{Additional Experiments and Experimental Details}
All models were constructed using PyTorch and PyTorch-Geometric. Each model was trained for 500 epochs on a training set consisting of 1000 graphs with 100 nodes each. All models were trained on 8 NVIDIA RTX A6000 GPUs. Additionally, all models were trained using the ADAM optimizer~\citep{kingma2015adam} provided in PyTorch and trained using a learning rate of 0.001. 

We note that since the GNN baselines output node states rather than an explicit primal--dual pair, we construct a surrogate dual variable for these baselines from the predicted primal variables by taking pairwise differences across edges in order to train against the primal--dual gap (i.e. the dual variable estimate is the projected variable $p_{ij} = h_i^{(\ell)} - h_j^{(\ell)}$ where $h^{(\ell)}$ are learned node features). 
These baselines test whether the gains of \name{} can be explained by generic message passing alone. GIN, GraphSAGE, GCN, and GAT are node-state message-passing architectures with different aggregation mechanisms, while EdgeConv uses pairwise edge messages based on local feature differences.

\subsubsection{Detailed dataset descriptions.} 
Here, we give more detailed exposition on how each graph dataset was constructed. 

\paragraph{Hierarchical Gaussian Clusters Benchmark}
To empirically evaluate \name{} on multi-scale geometric graphs, we introduce the Hierarchical Gaussian Benchmark.

The dataset $D = \{G_m, X_m\}_{m=1}^M$ consists of graphs and node features sampled from a recursive Gaussian hierarchy. Each node in the hierarchy at level $\ell$ has $b_\ell$ children. Starting from a root node $c_{\emptyset}=0$, children are sampled recursively according to
$$
    c_{u,j} \sim \mathcal{N}(c_u, s_\ell^2 I_d),
$$
where $u$ indexes a node at level $\ell-1$ and $s_\ell$ is the standard deviation associated with level $\ell$. At the leaves of the hierarchy, node features are sampled as
$$
    x_i \sim \mathcal{N}(c_\alpha, s_L^2 I_d),
$$
where $\alpha = (j_1,\ldots,j_L)$ denotes the path from the root to a leaf. Thus, each point has a hierarchical label corresponding to its full path in the latent tree, rather than a single flat cluster assignment. The total number of nodes is
$$
    |V| = n_{\mathrm{leaf}} \prod_{\ell=1}^L b_\ell,
$$
where $n_{\mathrm{leaf}}$ is the number of sampled points per leaf.

Given the sampled features, the graph topology is induced geometrically by constructing a $k$-nearest neighbor graph with Gaussian edge weights.

\paragraph{Image datasets}

%\chester{Image datasets We construct a graph over the latent feature space. We used all available data to construct the subgraphs, with n= $70,000$ nodes for MNIST and Fashion-MNIST, and $n= 60,000$ nodes for CIFAR-10. The graph was constructed as a $k$-nearest neighbor graph with Gaussian edge weights given by $w_{ij} = \exp(-4\lVert v_i-v_j\rVert^2/d_k(v_i)^2)$ , where $v_i$ are the latent variables for image $i$, and $d_k(v_i)$ is the distance in the latent space between $v_i$ and its $k$-th nearest neighbor. We used $k= 10$ in all experiments and symmetrize $W$ by replacing $W$ with $\frac{1}{2}(W + W^\top)$.}

For MNIST and Fashion-MNIST, we used variational autoencoders with 3 fully connected layers of sizes (784,400,20) and (784,400,30), respectively, followed by a symmetrically defined decoder. The autoencoder was trained for 100 epochs on each dataset. The autoencoder architecture, loss, and training are similar to \citet{Kingma2013AutoEncodingVB}. For CIFAR-10, we utilize publicly available SimCLR embeddings~\cite{simclr}.

For each dataset, we construct subgraphs over the latent feature space. We used all available data to construct the graph, with $n=70,000$ nodes for MNIST and Fashion-MNIST, and $n=60,000$ nodes for CIFAR-10. The graph was constructed as a $k$-nearest neighbor graph with Gaussian edge weights given by
\[w_{ij} =\exp\left( -4||x_i-x_j||^2/d_k(x_i)^2 \right),\]
where $x_i$ are the latent variables for image $i$, and $d_k(x_i)$ is the distance in the latent space between $x_i$ and its $k^{\rm th}$ nearest neighbor. We used $k=10$ in all experiments and symmetrize $W$. In order to generate graphs for training, we subsample graphs of $100$ nodes from the overall $k$-nearest neighbor graph. 

\subsubsection{Additional experiments on convex clustering}

In this section we provide additional results and illustrations of results presented in the main text. We provide further results on image datasets as well as an additional ablation on the edge memory component of \name{}, a comparison of the empirical convergence rate versus the theoretical convergence rate, and an accounting of the computational costs of \name{}.

\paragraph{Further results on image datasets}
In \Cref{tab:init-primal-dual-gap}, we demonstrate the effectiveness of \name for generating solutions with low primal--dual gap for each image dataset. Since all models were trained using the primal--dual gap, this measure directly measures whether the network produces a primal--dual pair that is close to satisfying the optimality conditions of the convex clustering problem. A lower initial gap is especially important for warm-starting subsequent solvers, since it indicates that the subsequent solver begins from a more informative primal--dual state rather than from a generic initialization. 

Across all datasets and both $\lambda$ levels, \name achieves the lowest mean primal--dual gap. Both GAT and GCN exhibit a lack of explicit edge-state propagation. Compared with GAT, the improvement of \name is particularly pronounced. Compared with GCN, \name consistently improves the initialization quality, reducing the mean gap by approximately $42\%$ and $22\%$ on MNIST for $\lambda=2.5$ and $\lambda=10$, respectively, and by approximately $39\%$ and $24\%$ on Fashion-MNIST. On CIFAR-10, all methods have larger variance, reflecting the greater difficulty and heterogeneity of the instances, but \name still gives the best average initialization for both regularization levels. These results support our claim that the PDHG-aligned node-edge architecture learns more useful primal--dual structure than node-only GNN baselines.

The ablations in \Cref{tab:init-primal-dual-gap} further clarify which parts of the architecture are responsible for the improvement. The ``Unrolled PDHG'' baseline tests whether simply embedding a fixed number of classical PDHG iterations is sufficient to produce a useful initialization. Although this baseline is algorithmically aligned with the target solver, it performs worse than \name in all cases and is especially unstable on Fashion-MNIST and CIFAR-10. This suggests that fixed unrolling alone does not provide enough adaptivity across graph instances; the learned message-passing parameters in \name are important for producing low-gap primal--dual states within a small number of layers.

The ``\name{} (- projection)'' ablation removes the explicit projection/normalization step in the edge update. This ablation isolates the role of maintaining dual feasibility through the PDHG-inspired edge map. Removing this projection substantially degrades the initialization quality. On MNIST, the full model reduces the mean gap relative to this ablation by approximately $64\%$ for $\lambda=2.5$ and $37\%$ for $\lambda=10$. On Fashion-MNIST, the corresponding reductions are approximately $70\%$ and $38\%$, while on CIFAR-10 they are approximately $12\%$ and $10\%$. These results indicate that the projection is not merely a cosmetic architectural constraint: it helps the network produce dual variables that are better aligned with the feasible dual geometry of convex clustering.

These ablations indicate that the performance of \name is not explained solely by using a node--edge architecture or by unrolling a known optimizer. The best results require both learned primal--dual message passing and the PDHG-inspired edge projection. This supports the central design principle of \name: preserving the algorithmic structure of PDHG while allowing the local update maps to adapt to the graph and feature distribution.

\Cref{fig:ssnal-warmstart-comparison} compares the downstream effect of different initializations on SSNAL across MNIST, Fashion-MNIST, and CIFAR-10. The left column reports the mean first-hit iteration: the first SSNAL iteration at which the primal--dual gap falls below a prescribed threshold. Thus, lower curves indicate that the initialization allows SSNAL to reach a certified solution in fewer iterations. Across all three datasets, \name yields the fastest or near-fastest convergence, with the largest gains appearing at stricter gap thresholds. The right column reports the primal--dual gap of the neural initialization before SSNAL refinement. \name produces consistently lower initial gaps than the GAT and GCN benchmarks, supporting the claim that explicitly maintaining edge-dual states yields more informative primal--dual warm starts. CIFAR-10 is more variable, but the same qualitative trend remains: \name gives a lower initial gap and reduces the number of SSNAL iterations needed for high-accuracy refinement.

\begin{table}[t]
\centering
\small
\setlength{\tabcolsep}{5pt}
\caption{Primal--dual gap at initialization on image-derived convex clustering instances. Entries report mean $\pm$ standard deviation over runs. Lower is better, and the best result for each dataset and weight $\lambda$ is bolded.}
\label{tab:init-primal-dual-gap}
    \begin{tabular}{llcc}
        \toprule
        \textbf{Dataset} & \textbf{Method}
        & $\boldsymbol{\lambda = 2.5}$
        & $\boldsymbol{\lambda = 10}$ \\
        \midrule
        \textbf{MNIST}
        & GAT      & $1931.7 \pm 1274.4$ & $2167.9 \pm 1544.4$ \\
        & GIN      & $15.0 \pm 4.5$ & $98.0 \pm 30.0$ \\
        & GCN      & $12.6 \pm 3.1$       & $89.0 \pm 24.9$ \\
        & GraphSAGE & $11.0 \pm 2.8$ & $82.0 \pm 23.0$ \\
        & EdgeConv  & $9.0 \pm 2.0$  & $74.0 \pm 18.0$ \\
        & Unrolled PDHG & $64.3 \pm 28.1$        & $143.2 \pm 59.3$ \\
        & \name (- projection) & $20.4 \pm 4.6$        & $110.4 \pm 41.3$ \\
        & \name & $\mathbf{7.3 \pm 1.3}$        & $\mathbf{69.5 \pm 14.8}$ \\
        \midrule
        \textbf{Fashion-MNIST}
        & GAT      & $700.6 \pm 541.8$ & $764.3 \pm 530.1$ \\
        & GIN      & $12.5 \pm 5.0$ & $95.0 \pm 31.0$ \\
        & GCN      & $10.0 \pm 3.5$ & $84.6 \pm 23.5$ \\
        & GraphSAGE & $8.8 \pm 3.2$ & $78.0 \pm 22.0$ \\
        & EdgeConv  & $7.3 \pm 2.4$ & $69.0 \pm 18.0$ \\
        & Unrolled PDHG & $8.7\times 10^{4} \pm 3.0\times 10^{3}$        & $9.4\times 10^{4} \pm 4.3\times 10^{3}$ \\
        & \name (- projection) & $20.4 \pm 19.1$        & $104.3 \pm 27.8$ \\
        & \name & $\mathbf{6.1 \pm 1.7}$ & $\mathbf{64.2 \pm 14.9}$ \\
        \midrule
        \textbf{CIFAR-10}
        & GAT      & $27467.9 \pm 3160.2$ & $27693.4 \pm 3256.8$ \\
        & GIN      & $3300.0 \pm 3200.0$ & $3450.0 \pm 3100.0$ \\
        & GCN      & $3098.7 \pm 3179.2$ & $3244.3 \pm 3017.8$ \\
        & GraphSAGE & $2950.0 \pm 2600.0$ & $3100.0 \pm 2550.0$ \\
        & EdgeConv  & $2750.0 \pm 1200.0$ & $2900.0 \pm 1150.0$ \\
        & Unrolled PDHG & $6.2\times 10^{6} \pm 1.4\times 10^{3}$        & $7.9\times 10^{6} \pm 1.7\times 10^{3}$ \\
        & \name (- projection) & $2978.5 \pm 74.7$        & $3057.3 \pm 83.1$ \\
        & \name & $\mathbf{2622.2 \pm 74.2}$ & $\mathbf{2741.3 \pm 79.3}$ \\
        \bottomrule
    \end{tabular}
\end{table}

\begin{figure}
    \centering
    \includegraphics[width=0.45\linewidth]{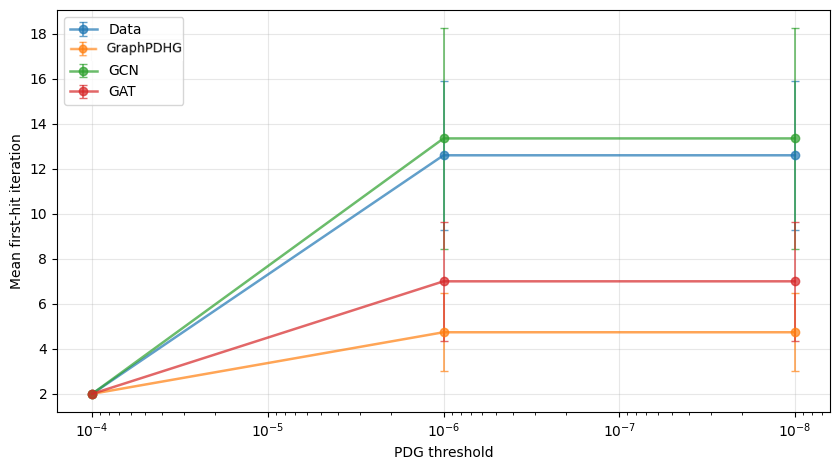}
    \includegraphics[width=0.425\linewidth]{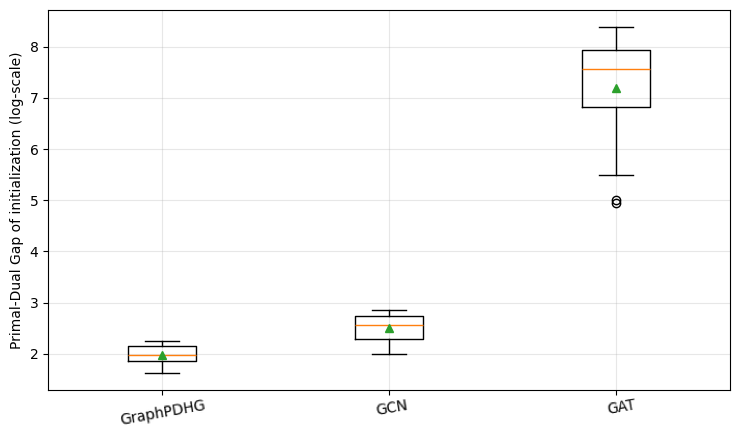}

    \includegraphics[width=0.45\linewidth]{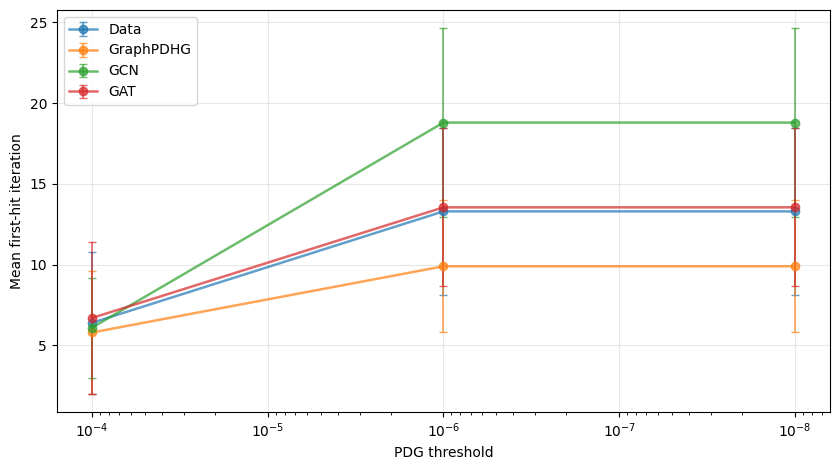}
    \includegraphics[width=0.425\linewidth]{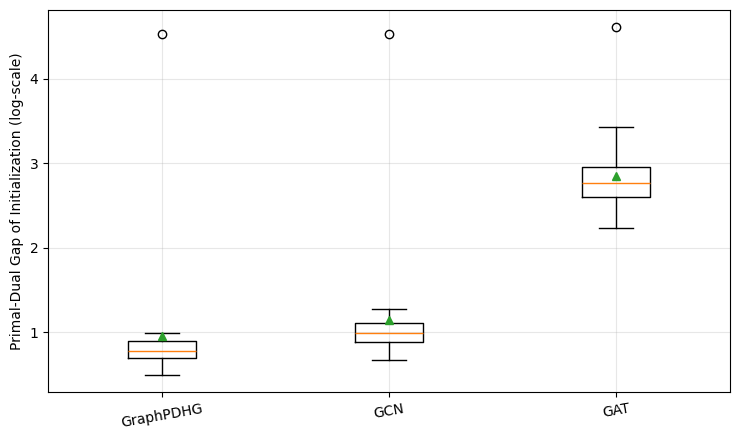}
    
    \includegraphics[width=0.45\linewidth]{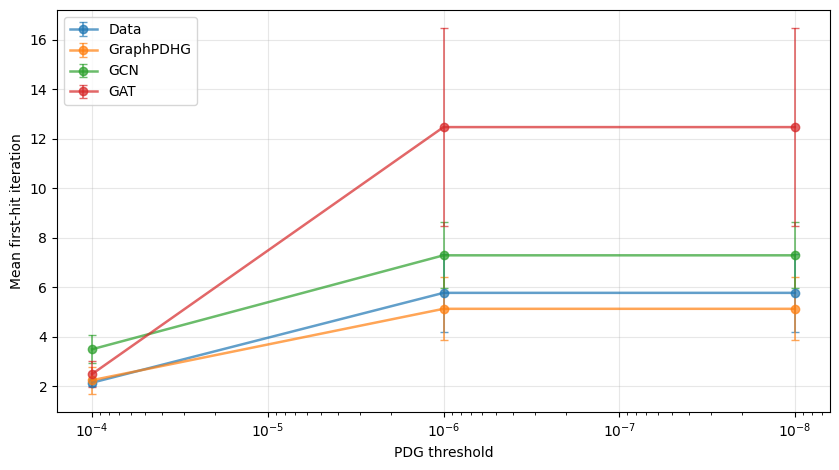}
    \includegraphics[width=0.425\linewidth]{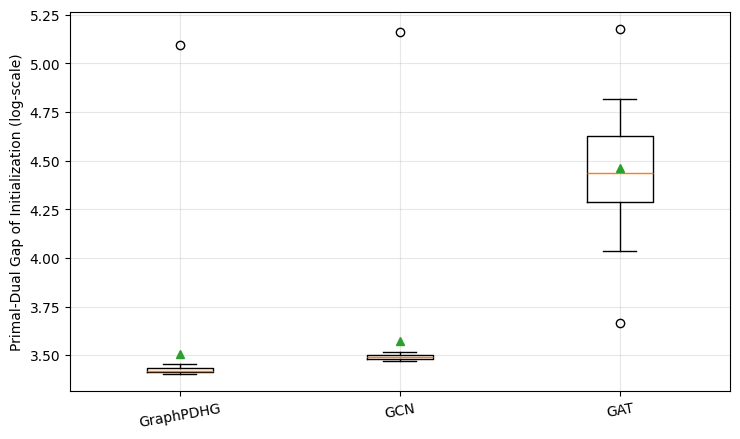}
    
    \caption{Comparing initializations. \textbf{Left Column: } Mean number of iterations of SSNAL to reach a primal--dual gap $\leq 10^{-6}$. \textbf{Right Column: } Primal-dual gap at initialization. \textbf{Top: }MNIST. \textbf{Middle: }Fashion-MNIST. \textbf{Bottom: }CIFAR-10.}
    \label{fig:ssnal-warmstart-comparison}
\end{figure} 

\paragraph{Ablations on edge memory of \name{}}

We performed additional ablations on the edge memory component of \name{} for several values of $\lambda$ on the hierarchical Gaussians dataset and record the median first hit for a threshhold of $10^{-6}$ below. By first hit, we mean the first iteration where each method reaches below the threshhold. Notice that without edge memory, \name{} have a  higher median first hit than with edge memory for lower $\lambda=0.001, 0.01$. They have a comparable median hit rate as $\lambda$ grows. However, this is because as $\lambda$ grows, the convex clustering objective collapses into a single point (the global mean) and most baselines can find this degenerate solution. Nevertheless, this ablation points to the importance of the edge component of \name{} (representing the dual variable) is essential for the performance of the method. 

\begin{table}[t]
    \centering
    \caption{Ablating the edge memory component of \name{}. We measure the median first hit times for a target primal-dual gap of $0.001$ after initializing SSNAL with \name{} with edge memory vs. \name{} without edge memory. Note that initializing SSNAL with \name{} without the edge component, the median first hit is higher than if we had initialized SSNAL with \name{} with an edge component. }
    \begin{tabular}{ccccc}
    \hline
         & $\lambda = 0.001$ & $\lambda=0.01$ & $\lambda=0.10$& $\lambda=1.0$\\
         \hline
        With edge memory & 16.0 & 41.0 & 164.0 & 176.0\\
        Without edge memory & 18.0 	& 53.0 & 161.0 & 182.0 	 \\
        \hline
    \end{tabular}
    \label{tab:edge-ablation}
\end{table}

\paragraph{Empirical convergence rate versus accelerated convergence rate}
In \Cref{tab:empirical-theoretical-gaps}, we compare the empirical primal-dual gap achieved by \name{} on a dataset of expander graphs (with $\kappa=2.53$), together with the accelerated and standard theoretical convergence rates for PDHG. Note that $k$ in \Cref{tab:empirical-theoretical-gaps} indicates the iteration number for PDHG and the corresponding processor layer for \name{}. 

We observe that on the expander graphs, \name{} reduces the median relative primal-dual gap below \(10^{-2}\) using only two processor layers, while canonical PDHG first reaches this threshold at $12$ iterations. \name{} is the top performing method through depth $14$, with the two methods becoming comparable by depth $16$. For the median full-Laplacian condition number \(\kappa=2.53\), the accelerated theoretical factor described in the main text is \(q_\kappa=0.228\), illustrating geometric decay compared with PDHG’s normalized \(1/k\) ergodic-rate shape.
We believe that the gains begin to plateau at larger values of $k$ because increasing $k$ increases the depth of the message-passing network, making optimization more difficult. In practice, GNNs are often shallow (with relatively few message-passing layers), as deeper architectures can be susceptible to oversmoothing and oversquashing (see \cite{}). Although these effects can be partially mitigated through residual connections \cite{}(which are also included in our architecture), training and optimizing deep GNNs remains challenging. We believe that these depth-related optimization issues limit additional gains as $k$ increases. 

\begin{table}[ht]
\centering
\caption{Empirical gaps and theoretical convergence rates on expanders with $\kappa = 2.53$.}
\begin{tabular}{rcccc}
\hline
$k$
& \begin{tabular}[c]{@{}c@{}}GraphPDHG\\ empirical gap\end{tabular}
& \begin{tabular}[c]{@{}c@{}}PDHG\\ empirical gap\end{tabular}
& \begin{tabular}[c]{@{}c@{}}Accelerated theory\\ $q_\kappa^k$\end{tabular}
& \begin{tabular}[c]{@{}c@{}}PDHG theory\\ $1/k$\end{tabular}
\\
\hline
1  & $1.90 \times 10^{-2}$ & $5.02 \times 10^{-1}$ & $2.28 \times 10^{-1}$  & $1.00$ \\
2  & $1.60 \times 10^{-3}$ & $3.26 \times 10^{-1}$ & $5.20 \times 10^{-2}$  & $5.00 \times 10^{-1}$ \\
4  & $6.14 \times 10^{-3}$ & $1.41 \times 10^{-1}$ & $2.71 \times 10^{-3}$  & $2.50 \times 10^{-1}$ \\
8  & $1.35 \times 10^{-3}$ & $3.06 \times 10^{-2}$ & $7.33 \times 10^{-6}$  & $1.25 \times 10^{-1}$ \\
12 & $1.48 \times 10^{-3}$ & $7.52 \times 10^{-3}$ & $1.99 \times 10^{-8}$  & $8.33 \times 10^{-2}$ \\
16 & $7.51 \times 10^{-4}$ & $6.89 \times 10^{-4}$ & $5.38 \times 10^{-11}$ & $6.25 \times 10^{-2}$ \\
\hline
\end{tabular}

\label{tab:empirical-theoretical-gaps}
\end{table}

\paragraph{Computational cost of \name{}}
Here, we provide a discussion of the computational cost of \name{}. 
 First, we note that our proposed \name{} method is linear time per forward pass (as is the vanilla PDHG algorithm). All baselines are also linear time per forward pass so there is no significant difference in their theoretical complexity. Empirically, we record the following inference and training times for each method.  

The training times recorded are for a training dataset of 1000 hierarchical Gaussian graphs with at most 100 nodes per graph and 15,000 training steps. The inference times recorded are for a single forward pass of each method on graphs with up to 100 nodes. 
\begin{table}[ht]
    \centering
     \caption{Training time and inference times for each method on hierarchical Gaussian graphs with 100 nodes. }
    \begin{tabular}{cccccccc}
\toprule
     & GraphPDHG & GCN & GAT\\
     \midrule
    Training time & 1h 20m &  1h 26m & 1h 12m \\
    Inference time (ms) & 4.367 &  2.977 & 5.602\\
    \bottomrule
\end{tabular}
   
    \label{tab:placeholder}
\end{table}

\subsection{Total Variation Denoising on Graphs}
%\sam{Chester: could you fill in the details here?}
%We include a total variation (TV) graph denoising problem as an additional example of the utility of \name{} on graph saddle point problems. 
%Here, we sample weighted PyGSP sensor graphs with 244 nodes in 2D and approximately 8-nearest-neighbor connectivity. The `signal', i.e. label, for each node is a piecewise-affine graph signal. Nodes are partitioned into four random Voronoi regions, and each region receives its own random affine function: $x_i^* = a_{r(i)}^\top p_i  + b_{r(i)}$. The resulting signal is standardized to zero mean and unit variance. The observation is $y_i = x_i^* + \varepsilon_i$, $\varepsilon \sim \mathcal{N}(0, 0.4)^2$. The prediction task is to esimate the clean scalar $x_i^*$ at every node from the noisy values $y_i$ and the graph structure. Region assignments and clean signals are not given to the graph. We impose convex TV and nonconvex capped-fusion regularization on \name{}.  Our results are summarized in \Cref{tab:tv-method-comparison}. 

We evaluate \name{} on graph-signal denoising to demonstrate its applicability beyond the original clustering experiments. We consider both convex total variation (TV) and nonconvex, nonsmooth capped-fusion regularization, and compare against classical solvers and parameter-matched neural baselines.

\paragraph{Signal generation.} We generate connected, weighted PyGSP sensor graphs with $n=244$ nodes on a randomly perturbed $12\times12$ grid in $[0,1]^2$, using PyGSP's symmetrized $k$-nearest-neighbor construction with $k=8$. To construct a clean signal, we sample four Voronoi centers uniformly in the unit square. Each region receives independent coefficients $a_r\sim\mathcal{N}(0,I_2)$ and $b_r\sim\mathcal{N}(0,0.5^2)$. For node coordinates $p_i$, we define
\begin{equation}
    s_i=a_{r(i)}^\top p_i+b_{r(i)},\quad
    x_i^\star=\frac{s_i-\bar{s}}
    {\sqrt{n^{-1}\sum_j(s_j-\bar{s})^2}},\quad
    y_i=x_i^\star+\varepsilon_i,
    \:\: \varepsilon_i\sim\mathcal{N}(0,0.4^2).
\end{equation}
Thus, each clean signal has zero mean and unit variance, and the observation noise has variance $0.16$. We use 512 training, 128 validation, and 256 test graphs. The prediction task is to recover the clean scalar signal from the noisy observations and graph. Clean signals serve as supervised training targets only; neither clean values nor region assignments are inputs at inference.

\paragraph{Reconstruction objectives.} For each undirected edge $e=(i,j)$, define the weighted incidence
operator $(Ku)_e=\sqrt{w_{ij}}(u_i-u_j)$. The two reconstruction objectives that we consider are
\begin{align}
    F_{\mathrm{TV}}(u)
    &=\tfrac12\|u-y\|_2^2+\lambda_{\mathrm{TV}}\|Ku\|_1,
    \label{eq:sensor-tv}\\
    F_{\mathrm{cap}}(u)
    &=\tfrac12\|u-y\|_2^2+
    \lambda_{\mathrm{cap}}\sum_{e\in E}\min\{|(Ku)_e|,c\}.
    \label{eq:sensor-capped}
\end{align}
The convex TV problem has the saddle-point formulation
\begin{equation}
    \min_u\max_{\|q\|_\infty\leq\lambda_{\mathrm{TV}}}
    \tfrac12\|u-y\|_2^2+\langle Ku,q\rangle.
\end{equation}
Capped fusion reduces the penalty on large jumps by saturating their contribution. 

\paragraph{Methods} We evaluate \name{} against the GNN methods described in \Cref{sec:experiments}. Additionally, we distinguish between \name{}-TV uses the TV operator throughout; \name{}-MM starts with a TV-style update and then applies four iterations of the MM (majorize-minimization) algorithm recomputed from the current reconstruction.
Note that \name{}-MM approximately solves the capped-fusion proximal problem 
$$
    \min_u \frac12\|u-y\|_2^2 +\lambda\sum_{e\in E}\min\{|(Ku)_e|,c\} 
$$
At each iteration of the proximal map, we construct an upper bound by linearizing the concave function $s\mapsto\min\{s,c\}$ at the current edge magnitudes. Defining $m_e^{(t)}=\mathbf{1}\{|(Ku^{(t)})_e|<c\}$, the resulting subproblem is
$$
    \min_u \frac12\|u-y\|_2^2 +\lambda\sum_{e\in E}m_e^{(t)}|(Ku)_e|.
$$
Each iteration of MM alternatives between applying the TV-style proximal map and updating the mask $m_e$. We also compare against vanilla PDHG with the TV operator as well as vanilla PDHG with MM updates. 

\paragraph{Results} We measure both the normalized mean squared error (NMSE) between the reconstructed signal output (from each method) and the clean signal and the signal-to-noise ratio gain (SNR gain) which measures the difference between the input noisy graph signal and the output signals from each method.  \name{}-MM achieved the lowest NMSE at 3.33e-02, a 79\% reduction from the noisy observation's 1.57e-01 and a 6.85 dB SNR improvement. \name{}-TV and GAT were close at 3.37e-02 and 3.42e-02, respectively. GraphSAGE, EdgeConv, and GCN were also competitive, with NMSE values between 3.55e-02 and 3.68e-02. For CPU inference runtime, GraphSAGE was the fastest graph-based model at 2.03 ms/graph, followed by GCN at 2.89 ms/graph. \name{}-TV and \name{}-MM required approximately 4.12 ms/graph, while EdgeConv and GAT required 5.05 and 5.43 ms/graph, respectively. PDHG-MM was considerably slower at 10.73 ms/graph.

\begin{table}[ht]
\centering
\caption{Graph signal denoising. Comparison of reconstruction methods. Uncertainties are computed as half the reported interval width.}
\label{tab:tv-method-comparison}
\begin{tabular}{lccc}
\toprule
Method & NMSE & SNR gain (dB) & CPU (ms) \\
\midrule
Noisy
& $3.922 \times 10^{-2} \pm 6.75 \times 10^{-4}$
& $0.00 \pm 0.00$
& $0.000$ \\
Heat
& $3.231 \times 10^{-2} \pm 6.85 \times 10^{-4}$
& $0.91 \pm 0.08$
& $0.816$ \\
PDHG-TV
& $1.897 \times 10^{-2} \pm 5.65 \times 10^{-4}$
& $3.17 \pm 0.15$
& $3.184$ \\
PDHG-MM
& $1.461 \times 10^{-2} \pm 9.10 \times 10^{-4}$
& $4.31 \pm 0.28$
& $11.373$ \\
GraphPDHG-TV
& $1.075 \times 10^{-2} \pm 6.20 \times 10^{-4}$
& $5.63 \pm 0.20$
& $4.354$ \\
GraphPDHG-MM
& $1.098 \times 10^{-2} \pm 5.80 \times 10^{-4}$
& $5.50 \pm 0.18$
& $4.367$ \\
MLP
& $3.779 \times 10^{-2} \pm 5.55 \times 10^{-4}$
& $0.17 \pm 0.025$
& $0.670$ \\
GCN
& $1.116 \times 10^{-2} \pm 6.90 \times 10^{-4}$
& $5.53 \pm 0.16$
& $2.977$ \\
GAT
& $1.077 \times 10^{-2} \pm 5.25 \times 10^{-4}$
& $5.55 \pm 0.155$
& $5.602$ \\
\bottomrule
\end{tabular}
\end{table}

\clearpage

\section{Proofs and additional results}
\begin{corollary}
	\label{cor:strong_convexity_gap_to_distance}
	Suppose that
	$G$ is $\mu$-strongly convex:
	$$
		G(v)\ge G(u)+\langle g, v-u\rangle+\frac{\mu}{2}\|v-u\|_F^2
		\qquad
		\forall u,v\in\mathbb{R}^{|V|\times d},\text{ and } g\in\partial G(u).
	$$
	Then
	$
		f(u):=G(u)+F(Ku)
	$
	is also $\mu$-strongly convex and admits a unique minimizer $u^\star$.
	Moreover, if $(u^\star,p^\star)$ is any saddle point of \eqref{eq:saddle_P}, then for all
	$u\in\mathbb{R}^{|V|\times d}$ and $p\in\mathbb{R}^{|E|\times d}$,
	\begin{equation}
		\label{eq:generic_gap_to_distance}
		\frac{\mu}{2}\|u-u^\star\|_F^2
		\le
		\mathcal L(u,p^\star)-\mathcal L(u^\star,p)
		\le
		\text{Gap}(u,p).
	\end{equation}
	From Theorem~\ref{thm:pdhg_convergence}, we have that the PDHG iterates satisfy
	\begin{equation}
		\label{eq:generic_ergodic_distance}
		\frac{\mu}{2}\|\bar u^{(N)}-u^\star\|_F^2
		\le
		\frac{1}{N}\left(
    		\frac{s}{2}\|u^{(0)}-u^\star\|_S^2
    		+
    		\frac{t}{2}\|p^{(0)}-p^\star\|_T^2
		\right),
	\end{equation}
	Then,
	$$
		\|\bar u^{(N)}-u^\star\|_F \le
		\left(
            \frac{s\|u^{(0)}-u^\star\|_S^2 + t\|p^{(0)}-p^\star\|_T^2}{\mu N}
		\right)^{1/2}.
	$$
\end{corollary}
\begin{proof}[Proof of \Cref{cor:strong_convexity_gap_to_distance}]
	Let $(u^\star,p^\star)$ be any saddle point. By the saddle-point/KKT condition,
	$$
		-K^\top p^\star \in \partial G(u^\star).
	$$
	Applying strong convexity of $G$ at $u^\star$ with subgradient $-K^\top p^\star$ gives
	$$
		G(u)-G(u^\star)
		\ge
		\langle -K^\top p^\star,\;u-u^\star\rangle
		+
		\frac{\mu}{2}\|u-u^\star\|_F^2.
	$$
	Therefore,
	\begin{align*}
		\mathcal L(u,p^\star)-\mathcal L(u^\star,p^\star)
		  & =
		G(u)-G(u^\star)+\langle K(u-u^\star),p^\star\rangle \\
		  & \ge
		\frac{\mu}{2}\|u-u^\star\|_F^2.
	\end{align*}
	Since $p^\star$ maximizes $\mathcal L(u^\star,p)$, we also have
	$$
		\mathcal L(u^\star,p^\star)\ge \mathcal L(u^\star,p)
		\qquad \forall p,
	$$
	so
	$$
		\frac{\mu}{2}\|u-u^\star\|_F^2
		\le
		\mathcal L(u,p^\star)-\mathcal L(u^\star,p).
	$$
	For the upper bound, use
	$$
		\mathcal L(u,p^\star)\le \sup_q \mathcal L(u,q)=f(u),
		\qquad
		\mathcal L(u^\star,p)\ge \inf_v \mathcal L(v,p)=d(p),
	$$
	which implies
	$$
		\mathcal L(u,p^\star)-\mathcal L(u^\star,p)
		\le
		f(u)-d(p)
		=
		\text{Gap}(u,p).
	$$
	Finally, apply Theorem~\ref{thm:pdhg_convergence} with
	$u=\bar u^{(N)}$ and $p=\bar p^{(N)}$.
\end{proof}

\subsection{Proof of universality for \name} \label{sec:method_universality}

We first demonstrate that the message passing layer has the capacity to exactly recover one iteration of the PDHG algorithm, and subsequently inherits certain finite-layer convergence guarantees implied by Theorem~\ref{thm:pdhg_convergence}. This notion of algorithmic alignment has been previously shown in related work~\cite{he2025pdnar} to be critical for generalization in NAR tasks.
\begin{lemma}[Representation]\label{lem:representation}
    Let $\mathcal{A}_\text{PDHG}(\sigma, \tau)$ denote the operator defining one iteration of the preconditioned PDHG algorithm with fixed step sizes $\sigma, \tau > 0$. Assume the preconditioners have the form $S = \tau^{-1}I$ and $T = \sigma^{-1} I$. There exists a parameter configuration $\theta^* = \{A_1, A_2, B_1, B_2, B_3\}$ and a specific activation function for the message passing layers $\Phi_\theta$ such that, if initialized with $h^{(0)} = u^{(0)}$ and $e^{(0)} = p^{(0)}$, the layer outputs exactly recover PDHG iterates:
    $$
        h^{(\ell)} = u^{(\ell)}, \quad e^{(\ell)} = p^{(\ell)} \quad \forall \ell \geq 1.
    $$
\end{lemma}
\begin{proof}
    We proceed by constructively matching the update equations of the PDHG algorithm with the architecture of the proposed GNN layer. In particular, we observe that the GNN equations exactly match the PDHG updates under the following configuration $\theta^*$:
    \paragraph{Edge Update: } Set $A_1 = I$, $A_2 = \sigma \cdot I$, and $\phi_{\text{edge}} = \text{Proj}_{B(\lambda)}(\cdot)$.
    \paragraph{Node Update: } Set $B_1 = \frac{1}{1 + \tau}I$, $B_2 = \frac{\tau}{1 + \tau}I$, $B_3 = -\frac{\tau}{1+\tau}I$. \\

    \noindent Under this configuration, $\theta^*$, the outputs $(h^{(\ell + 1)} , e^{(\ell + 1)}) = (u^{(\ell + 1)}, p^{(\ell + 1)})$.
\end{proof}
As a consequence of the above lemma, PDHG-derived message passing inherits the approximation and convergence guarantees of the primal--dual approximation algorithm it learns, as detailed in Theorem~\ref{thm:universality}. 
\begin{proof}[Proof of \Cref{thm:universality}]
    The proof follows from the representation lemma, Lemma~\ref{lem:representation} and standard results on the convergence rate of the PDHG algorithm. 

    If $\tau\sigma \lVert K \rVert^2 < 1$, the ergodic average of the iterates $(\bar{u}_L, \bar{p}_L) = \frac{1}{L} \sum_{k=1}^L (u_k, p_k)$ satisfies the following bound on the primal--dual gap:
    \begin{equation}
        \text{Gap}(\bar{u}_L, \bar{p}_L) \leq \frac{1}{L} \left(\frac{\lVert u^{(0)} - u^*\rVert^2}{2\tau} + \frac{\lVert p^{(0)} - p^*\rVert^2}{2\sigma}\right),
    \end{equation}
    where $(u^*, p^*)$ is a saddle point of the problem. Let $R^2 =\frac{\lVert u^{(0)} - u^*\rVert^2}{2\tau} + \frac{\lVert p^{(0)} - p^*\rVert^2}{2\sigma}$. The bound becomes
    $$
        \text{Gap}_L \leq \frac{R^2}{L}.
    $$
    To ensure the Gap is at most $\epsilon$, we require
    $
        \frac{R^2}{L} \leq \epsilon \implies L \geq \frac{R^2}{\epsilon}.
    $
    Choose $\tau=\sigma=\frac{1}{1+\lVert K \rVert^2}$. Then $\tau\sigma \lVert K\rVert^2<1$, and $R^2$ scales with $\lVert K \rVert^2$. Thus, there exists a constant $C$ depending only on initialization such that if 
    $$
        L = \frac{C\cdot \lVert K \rVert^2}{\epsilon}
    $$
    the network achieves $\epsilon$-optimality in the primal--dual gap.
\end{proof}

\subsection{Analysis of linearized \name}

\begin{proof}[Proof of \Cref{prop:memorylessdynamics}]
    We construct a specific parameterization. Consider the memoryless regime. Let $A_1 = 0$, and for simplicity let $B_1 = I_d$. Write down the error matrix using the terms in Remark~\ref{rem:linear_regime}:
    $$
        \mathcal{E}^{(\ell+1)} = \mathcal{E}^{(\ell)} + K^\top (K\mathcal{E}^{(\ell)}A_2^\top )B_3^\top  = \mathcal{E}^{(\ell)} + \Delta \mathcal{E}^{(\ell)} W
    $$
    Where $W = A_2^\top B_3^\top \in \mathbb{R}^{d\times d}$. Choose $W = -\frac{1}{\lambda_k}I_d$ so
    $$
        \mathcal{E}^{(\ell+1)} = \mathcal{E}^{(\ell)} - \Delta\mathcal{E}^{(\ell)} \frac{1}{\lambda_k}I_d
    $$
    Left multiplying by $\psi_k^\top$ yields
    \begin{align*}
        \psi_k^\top\mathcal{E}^{(\ell+1)} &= \psi_k^\top\mathcal{E}^{(\ell)} - \psi_k^\top\Delta\mathcal{E}^{(\ell)} \frac{1}{\lambda_k}I_d \\
        &= \psi_k^\top\mathcal{E}^{(\ell)} - \psi_k^\top\mathcal{E}^{(\ell)}  = 0
    \end{align*}    
    Since $A_2$ and $B_3$ are learnable $d\times d$ matrices, we can realize $W = - \frac{1}{\lambda_k}I_d$ by e.g. $A_2 = I_d$, $B_3 = - \frac{1}{\lambda_k}I_d$.    
\end{proof}

\begin{proof}[Proof of \Cref{prop:memorydynamics}]
    From the linearized update equations with $A_1 = \gamma I_d$ and $B_1 = I_d$, we have:
    \begin{align}
        e^{(\ell+1)} &= \gamma e^{(\ell)} + K \mathcal{E}^{(\ell)} A_2^\top \label{eq:mom_edge} \\
        \mathcal{E}^{(\ell+1)} &= \mathcal{E}^{(\ell)} + K^\top e^{(\ell+1)} B_3^\top \label{eq:mom_node}
    \end{align}
    From \eqref{eq:mom_node}, we isolate the term $K^\top e^{(\ell+1)} B_3^\top = \mathcal{E}^{(\ell+1)} - \mathcal{E}^{(\ell)}$. Shifting the index back by one yields $K^\top e^{(\ell)} B_3^\top = \mathcal{E}^{(\ell)} - \mathcal{E}^{(\ell-1)}$.
    Multiply \eqref{eq:mom_edge} by $K^\top$ on the left and $B_3^\top$ on the right:
    $$
        K^\top e^{(\ell+1)} B_3^\top = \gamma (K^\top e^{(\ell)} B_3^\top) + K^\top (K \mathcal{E}^{(\ell)} A_2^\top) B_3^\top.
    $$
    Substituting the node difference expressions derived above:
    $$
        (\mathcal{E}^{(\ell+1)} - \mathcal{E}^{(\ell)}) = \gamma (\mathcal{E}^{(\ell)} - \mathcal{E}^{(\ell-1)}) + \Delta \mathcal{E}^{(\ell)} (A_2^\top B_3^\top).
    $$
    Define $M = -(A_2^\top B_3^\top)$. We recover the Polyak Heavy Ball update:
    $$
        \mathcal{E}^{(\ell+1)} = \mathcal{E}^{(\ell)} + \gamma (\mathcal{E}^{(\ell)} - \mathcal{E}^{(\ell-1)}) - \Delta \mathcal{E}^{(\ell)} M.
    $$    
\end{proof}

\begin{proof}[Proof of \Cref{cor:heavyball}]
    Consider the simplification $M = -A_2^\top B_3^\top = \eta I_d$. Then,
    $$
        \mathcal{E}^{(\ell+1)} = \mathcal{E}^{(\ell)} + \gamma (\mathcal{E}^{(\ell)} - \mathcal{E}^{(\ell-1)}) - \eta \Delta \mathcal{E}^{(\ell)}.
    $$
    Let $\Delta \psi_k = \lambda_k \psi_k$ such that $\lVert \psi_k \rVert_2 = 1$. Define the coefficient
    $$
        c_k^{(\ell)} = \psi_k^\top \mathcal{E}^{(\ell)} \in \mathbb{R}^d
    $$
    Left-multiply the recursion by $\psi_k^\top$:
    \begin{align*}
        c_k^{(\ell + 1)} &= c_k^{(\ell)} + \gamma (c_k^{(\ell)} - c_k^{(\ell - 1)}) - \eta \psi_k^\top \Delta \mathcal{E}^{(\ell)} \\
        &= c_k^{(\ell)} + \gamma (c_k^{(\ell)} - c_k^{(\ell - 1)}) - \eta \lambda_k \psi_k^\top \mathcal{E}^{(\ell)} \\
        &= (1 + \gamma - \eta \lambda_k)c_k^{(\ell)} - \gamma c_k^{(\ell-1)}
    \end{align*}
    Equivalently, each coordinate of $c_k^{(\ell)}$ satisfies
    $$
        c_{k,j}^{(\ell + 1)} = (1 + \gamma - \eta \lambda_k)c_{k,j}^{(\ell)} - \gamma c_{k,j}^{(\ell-1)}
    $$
    This recurrence has characteristic polynomial:% (see \cite{}):
    \begin{equation}
        z^2 - (1 + \gamma - \eta \lambda_k)z  + \gamma = 0.
    \end{equation}
    Finally, \cite{Polyak1964SomeMO} shows that choosing 
    $$ 
        \eta^\star=\frac{4}{(\sqrt{\lambda_{\max}}+\sqrt{\lambda_{\min}})^2}, \quad \gamma^\star=\left(\frac{\sqrt{\lambda_{\max}}-\sqrt{\lambda_{\min}}}{\sqrt{\lambda_{\max}}+\sqrt{\lambda_{\min}}}\right)^2 
    $$ 
    yields a worst-case error which decays at the accelerated rate 
    $$ 
        \max_k \|c_k^{(\ell)}\|_2 \le C \left(\frac{\sqrt{\kappa}-1}{\sqrt{\kappa}+1}\right)^\ell 
    $$ 
    for some constant $C$.
\end{proof}

In the following Corollary, we tighten the analysis to consider the more general case where $M$ is symmetric and positive definite. 

\begin{corollary}
\label{cor:heavyball_joint_kappa}
    Assume the setting of Proposition~\ref{prop:memorydynamics}, and suppose in addition that
    \(M \in \mathbb R^{d\times d}\) is symmetric positive definite. Let
    $
        \Delta \psi_k = \lambda_k \psi_k,
    $
    and
    $
        \lambda_1 < \lambda_2 \le \cdots \le \lambda_{|V|},
    $
    and let
    $
        M q_j = \mu_j q_j,
    $
    and
    $    \mu_j \ge 0, j=1,\dots,d.
    $
    Expand the error in the joint eigenbasis:
    $$
        \mathcal E^{(\ell)}
        =
        \sum_{k=1}^{|V|}
        \sum_{j=1}^{d}
        c_{k,j}^{(\ell)} \psi_k q_j^\top .
    $$
    Then each joint mode satisfies the scalar recurrence
    $$
        c_{k,j}^{(\ell+1)}
        =
        \bigl(1+\gamma-\lambda_k\mu_j\bigr)c_{k,j}^{(\ell)}
        -
        \gamma c_{k,j}^{(\ell-1)}.
    $$
    Define the joint condition number
    $
        \kappa_{\mathrm{joint}} := \frac{\max_{k,j} \lambda_k\mu_j}{\min_{k, j} \lambda_k\mu_j}.
    $
    If
    $
        \gamma^\star
        :=
        \left(
            \frac{\sqrt{\kappa_{\mathrm{joint}}}-1}
                 {\sqrt{\kappa_{\mathrm{joint}}}+1}
        \right)^2.
    $
    Then there exists a constant \(C>0\), such that
    $$
        \max_{\substack{k,j}}
        |c_{k,j}^{(\ell)}|
        \le
        C \left(\frac{\sqrt{\kappa_{\mathrm{joint}}}-1}
             {\sqrt{\kappa_{\mathrm{joint}}}+1} \right)^\ell.
    $$
\end{corollary}
\begin{proof}
    By Proposition~\ref{prop:memorydynamics},
    $$
        \mathcal E^{(\ell+1)}
        =
        \mathcal E^{(\ell)}
        +
        \gamma\bigl(\mathcal E^{(\ell)}-\mathcal E^{(\ell-1)}\bigr)
        -
        \Delta \mathcal E^{(\ell)} M.
    $$
    Expand
    $$
        \mathcal E^{(\ell)}
        =
        \sum_{k=1}^{|V|}
        \sum_{j=1}^{d}
        c_{k,j}^{(\ell)} \psi_k q_j^\top,
    $$
    where \(\Delta\psi_k=\lambda_k\psi_k\) and \(Mq_j=\mu_j q_j\). Then
    $$
        \Delta(\psi_k q_j^\top)M
        =
        \lambda_k\mu_j\psi_k q_j^\top,
    $$
    and the recurrence decouples over all of the modes (see the derivation in the previous corollary):
    $$
        c_{k,j}^{(\ell+1)}
        =
        \bigl(1+\gamma-\lambda_k\mu_j\bigr)c_{k,j}^{(\ell)}
        -
        \gamma c_{k,j}^{(\ell-1)}.
    $$
    Thus the \((k,j)\)-th mode is governed by the heavy-ball recurrence. 
    For the scalar heavy-ball method on a quadratic, the optimal worst-case asymptotic factor is
    $$
        \rho^\star
        =
        \frac{\sqrt{\kappa_{\mathrm{joint}}}-1}
             {\sqrt{\kappa_{\mathrm{joint}}}+1},
        \qquad
        \kappa_{\mathrm{joint}} := \frac{\max_{k,j} \lambda_k\mu_j}{\min_{k, j} \lambda_k\mu_j}.
    $$
    attained by
    $$
        \gamma^\star
        =
        \left(
            \frac{\sqrt{\kappa_{\mathrm{joint}}}-1}
                 {\sqrt{\kappa_{\mathrm{joint}}}+1}
        \right)^2.
    $$
    Applying this uniformly yields
    $
        |c_{k,j}^{(\ell)}|
        \le
        C_{k,j}(\rho^\star)^\ell.
    $
\end{proof}
\begin{proof}[Proof of \Cref{cor:layerbound}]
    From Proposition~\ref{prop:memorydynamics}, there exists a parameter configuration such that the error dynamics follow the accelerated heavy-ball method. The asymptotic linear convergence rate $\rho$ of this recurrence is given by:
    $$
        \rho = \frac{\sqrt{\kappa} - 1}{\sqrt{\kappa} + 1} = 1 - \frac{2}{\sqrt{\kappa} + 1}.
    $$
    We need to show the following bound:
    $$
        \lVert \mathcal{E}^{(\ell)}\rVert_2 \leq O(\kappa) \lVert \mathcal{E}^{(0)} \rVert_2 
    $$

    Let $\epsilon > 0$. We want $\ell$ such that $\frac{\|\mathcal{E}^{(\ell)}\|}{\|\mathcal{E}^{(0)}\|} < \epsilon$. From the inequality shown above, we just need the minimum $\ell$ such that $\rho^\ell < \epsilon$. Since $\rho, \epsilon < 1$, $\ell \log (1/\rho) \geq \log (1/ \epsilon)$ and $\ell \geq \log (1/ \epsilon) \cdot \frac{1}{\log (1/\rho)}$.

    Therefore, for any iteration $\ell$ such that $\ell \geq \log (1/ \epsilon) \cdot \frac{1}{\log (1/\rho)}$,  $\frac{\|\mathcal{E}^{(\ell)}\|}{\|\mathcal{E}^{(0)}\|} < \epsilon$. By \Cref{lem:representation}, there is a parameter configuration of \name such that each layer implements one iteration of PDHG. Therefore,  $\ell = \log (1/ \epsilon) \cdot \frac{1}{\log (1/\rho)}$-layers of \name can implement $\ell$ iterations of PDHG, so we get that the error at the $\ell$-th layer of GNN is $\frac{\|\mathcal{E}^{(\ell)}\|}{\|\mathcal{E}^{(0)}\|} < \epsilon$.
\end{proof}

\end{document}